\documentclass{article}

\usepackage[preprint]{neurips_2026}
\usepackage{amsmath}
\usepackage{amssymb}
\usepackage{xcolor} %
\definecolor{VN}{rgb}{0.278, 0.4, 0.898}
\definecolor{AN}{rgb}{0.898, 0.624, 0.278}
\definecolor{BLN}{rgb}{0.418, 0.894, 0.278}
\definecolor{DMN}{rgb}{0.902, 0.275, 0.424}
\definecolor{SMN}{rgb}{0.537, 0.275, 0.894}
\definecolor{SN}{rgb}{0.894, 0.827, 0.271}
\definecolor{MN}{rgb}{0.275, 0.898, 0.827}
\definecolor{CCN}{rgb}{0.075, 0.624, 1.0}

\usepackage[utf8]{inputenc} %
\usepackage[T1]{fontenc}    %
\usepackage{hyperref}       %
\usepackage{url}            %
\usepackage{booktabs}       %
\usepackage{amsfonts}       %
\usepackage{nicefrac}       %
\usepackage{microtype}      %
\usepackage{graphicx}
\usepackage[ruled,linesnumbered]{algorithm2e}
\usepackage{wrapfig}   %

\usepackage{amsthm}
\theoremstyle{definition}
\newtheorem{proposition}{Proposition}

\title{Information Bottleneck-Guided Adaptive Hypergraph Transformer for Brain Disease Diagnosis}
\author{Jingxi Feng \quad Xudong Chen \quad Yifan Zhang \quad Heming Xu \\
Hongcheng Han \quad Xijing Wang \quad Dong Zhang \quad Shaoyi Du}

\hypersetup{pdfauthor={Jingxi Feng, Xudong Chen, Yifan Zhang, Heming Xu, Hongcheng Han, Xijing Wang, Dong Zhang, Shaoyi Du}, pdftitle={Information Bottleneck-Guided Adaptive Hypergraph Transformer for Brain Disease Diagnosis}}

\begin{document}
\maketitle
\begin{abstract}
Exploring high-order correlations and long-range dependencies in brain networks holds significant value for both neuroscience research and clinical diagnosis. However, previous studies have lacked a unified integration of high-order and long-range dependency information in brain networks, and there is substantial redundancy behind various types of information. These issues limit their effectiveness in the diagnosis of brain diseases. To address this, we propose an \textbf{I}nformation \textbf{B}ottleneck-Guided \textbf{A}daptive \textbf{H}yper\textbf{G}raph \textbf{T}ransformer (IBAHGT). By incorporating the information bottleneck (IB) principle, this approach enables adaptive learning of high-order correlations and both short- and long-range dependencies within a unified framework for brain network analysis, achieving high-precision brain disease diagnosis. IBAHGT consists of three key components: an information bottleneck-guided adaptive hypergraph convolution, which introduces a novel hypergraph information bottleneck (HIB) principle to adaptively learn hypergraph message-passing weights between nodes and hyperedges, optimizes information flow and captures high-order information in brain networks that is maximally informative and minimally redundant (MIMR).
The Transformer encoder captures global information within brain networks through the attention mechanism, specifically modeling short- and long-range dependencies. An information bottleneck-guided node-level adaptive fusion employs the IB principle to learn independent weights for each node, facilitating the fine-grained integration of high-order information and global information to obtain an efficient representation for downstream tasks. Extensive experiments demonstrate that the proposed method outperforms current state-of-the-art methods and can identify biomarkers for clinical applications.
\end{abstract}
  
\section{Introduction}

The brain network describes the interactions and information processing patterns between regions of interest (ROIs) \cite{avena2018communication}. Brain network analysis, which explores interactions between ROIs, reveals the brain's functional organization and is crucial for understanding the pathogenesis of neurological diseases and their early intervention \cite{liu2017complex,supekar2008network}.
As shown in Fig.\ref{11}, many studies have shown the presence of extensive high-order correlations \cite{benson2016higher}, along with short- and long-range dependencies, within brain networks \cite{park2013structural,barttfeld2011big}.
High-order correlations reflect the joint activation of multiple ROIs during functional activities such as cognitive tasks,  and are crucial for understanding the activity patterns of the brain’s neural system \cite{von1994correlation}. Furthermore, during functional activities, the brain selectively integrates information through the coordinated work of high-order neural circuits to improve task execution efficiency and accuracy \cite{luna2004emergence}. Short-range dependencies reflect the interactions between ROIs in the neighboring space \cite{dajani2016local}, while long-range dependencies reflect long-distance communication between ROIs and play a critical role in understanding brain functions \cite{deco2021rare} and communication \cite{seguin2023brain}. Therefore, it is necessary to simultaneously capture the high-order correlations and long-range dependencies in the brain network for unified analysis.
\begin{wrapfigure}{r}{0.5\textwidth} 
    \centering  %
    \includegraphics[width=0.5\textwidth]{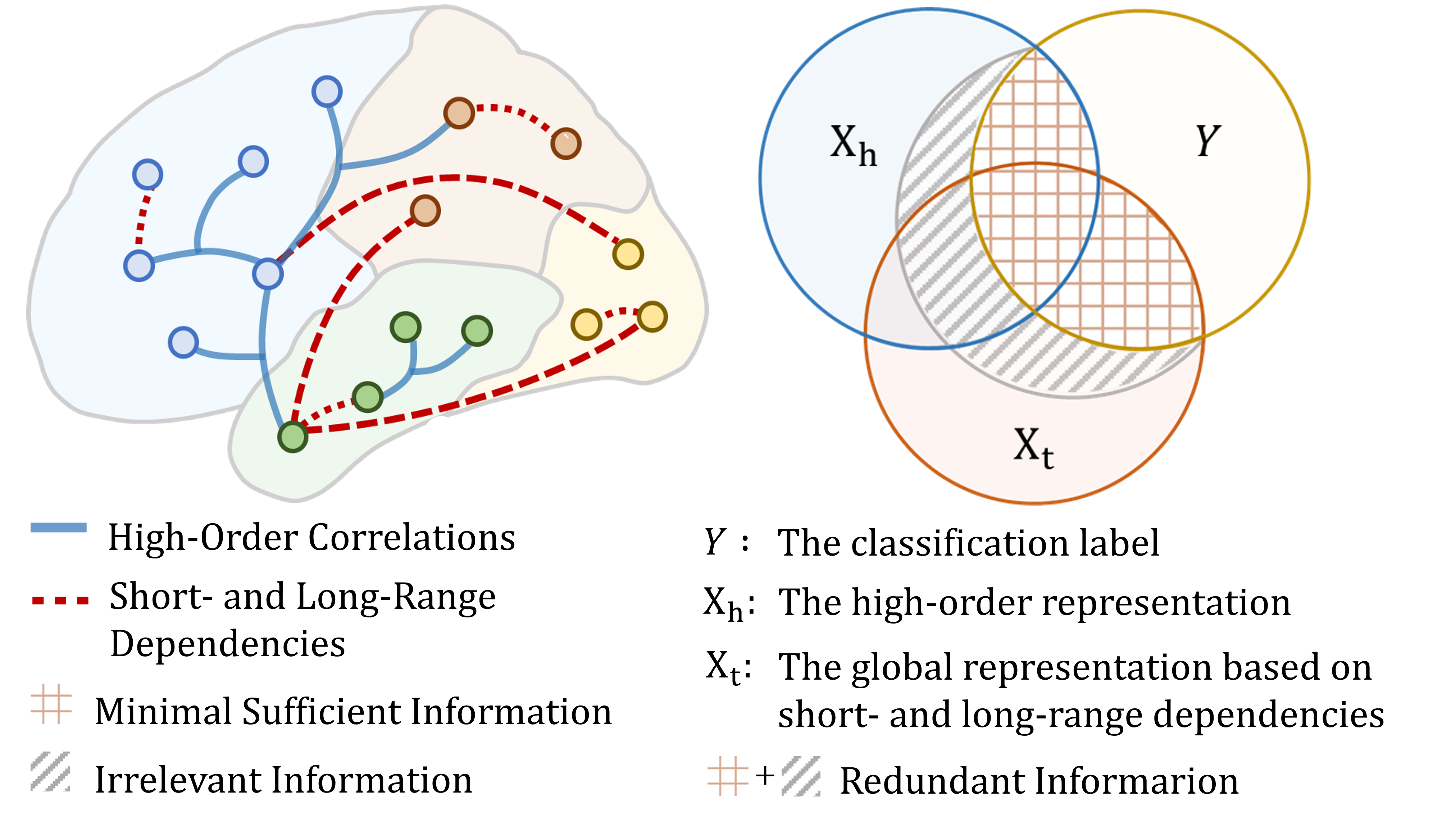} 
    \caption{The brain network contains extensive high-order correlations as well as short-range and long-range dependencies. Multi-source information fusion, guided by the information bottleneck, aims to optimize the fused representation by capturing the minimal sufficient information from high-order representation $\mathbf{X}_h$ and global representation $\mathbf{X}_t$ for predicting the classification label $Y$.}  %
    \label{11}
\end{wrapfigure}

Many studies use graph learning methods for brain network analysis \cite{bessadok2022graph, qu2021brain}, modeling the brain network as a graph and capturing pairwise dependencies in feature interactions through Graph Neural Networks (GNNs).
However, these methods overlook the crucial high-order correlations and fail to capture the collaborative interactions among multiple ROIs involved in brain functions. Hypergraphs, a mathematical model that can express high-order correlations, have been increasingly used in brain network analysis \cite{niu2023applications}. For example, FC-HAT effectively captures high-order information in brain networks through hypergraph attention networks \cite{ji2022fc}, thereby improving the accuracy of brain disease diagnosis.

GNNs and Hypergraph Neural Networks (HGNNs) \cite{feng2019hypergraph,bai2021hypergraph}, which aggregate neighborhood information through message passing, struggle to capture long-range dependencies in brain networks. Unlike the local information aggregation mechanisms of GNNs and HGNNs, the attention mechanism in Transformer can capture long-range dependencies between nodes \cite{vaswani2017attention}. Brain network Transformer \cite{kan2022brain} has introduced Transformer into brain network analysis to learn long-range dependencies between ROIs, achieving promising performance in brain disease diagnosis. 

However, these methods fail to integrate the high-order correlations and long-range dependencies according to the specific characteristics of each ROI. Furthermore, as shown in Fig.\ref{11}, both the high-order information captured by general hypergraph convolution and the information captured based on different dependencies contain substantial redundancy, which dilutes the disease-related discriminative information. Thus, it is crucial to condense the information during the extraction of high-order features and the fusion of multi-source information to achieve efficient representations.

To address these issues, we propose an information bottleneck-guided adaptive hypergraph Transformer method that effectively captures the MIMR high-order correlations as well as short- and long-range dependencies in the brain network.
Specifically, first, an information bottleneck-guided adaptive hypergraph convolution introduces an innovative hypergraph information bottleneck (HIB) principle to optimize the message passing process, achieving the goal of capturing MIMR high-order information.
Secondly, the Transformer encoder captures short- and long-range dependencies between ROIs through the self-attention mechanism. Finally, an information bottleneck-guided node-level adaptive fusion method considers the unique roles and needs of each ROI, and adaptively learns the weights for each node (i.e., ROI) to fuse representations that contain high-order information and global information.
In summary, this paper makes three main contributions:
\begin{itemize}
    \item[1)]  We propose an information bottleneck-guided adaptive hypergraph Transformer method that captures high-order correlations as well as short- and long-range dependencies in the brain network within a unified framework, achieving high-precision brain disease diagnosis.
    \item[2)] We introduce an information bottleneck-guided adaptive hypergraph convolution method. By incorporating the HIB principle, it adaptively learns the message passing weights between nodes and hyperedges, capturing MIMR high-order information in the brain network.
    \item[3)] We propose an information bottleneck-guided node-level adaptive fusion method that fine-grained integrates high-order and global information in the brain network. Further optimization through the IB principle effectively eliminates redundancy during multi-source information fusion, resulting in efficient representations.
    
\end{itemize}

\section{Related work}
\subsection{Deep Learning-Based Brain Network Analysis}
GNNs are commonly used deep learning methods for graph-structured brain networks learning \cite{yang2023mapping,ye2023rh}. BrainGNN \cite{li2021braingnn} introduces a novel ROI-aware graph convolution (Ra-GConv) layer for functional brain network analysis. However, graph-based methods overlook the high-order correlations that are widespread in brain networks. Recent studies have incorporated hypergraph learning methods into brain network analysis \cite{ji2022fc}. For example, dwHGCN \cite{wang2023dynamic} learns high-order correlations in brain networks through dynamic hypergraph learning with learnable hyperedge weights.
Graph- or hypergraph-based methods are unable to effectively capture long-range dependencies between ROIs. BrainNetTF \cite{kan2022brain} introduces a Transformer encoder to learn long-range dependencies and has demonstrated outstanding performance in brain disease diagnosis tasks.
However, the above methods fail to provide a unified analysis that incorporates both the high-order correlations as well as short- and long-range dependencies present in brain networks.

\subsection{Information Bottleneck}
The Information Bottleneck (IB) \cite{tishby2000information} is a data compression technique aimed at encouraging representations derived from raw data to capture the maximum relevant information about the target, while excluding extraneous information unrelated to the prediction task. Deep learning methods based on the IB principle have been widely applied in various fields, such as computer vision \cite{luo2019significance,peng2018variational} and natural language processing \cite{wang2020learning}. Recently, BrainIB \cite{zheng2024brainib} has incorporated the IB principle into brain network analysis, effectively identifying disease-specific prominent brain network connections, thus enabling interpretable brain disease diagnosis.
    
\section{Method}

\begin{figure}[t]
\includegraphics[width=\textwidth]{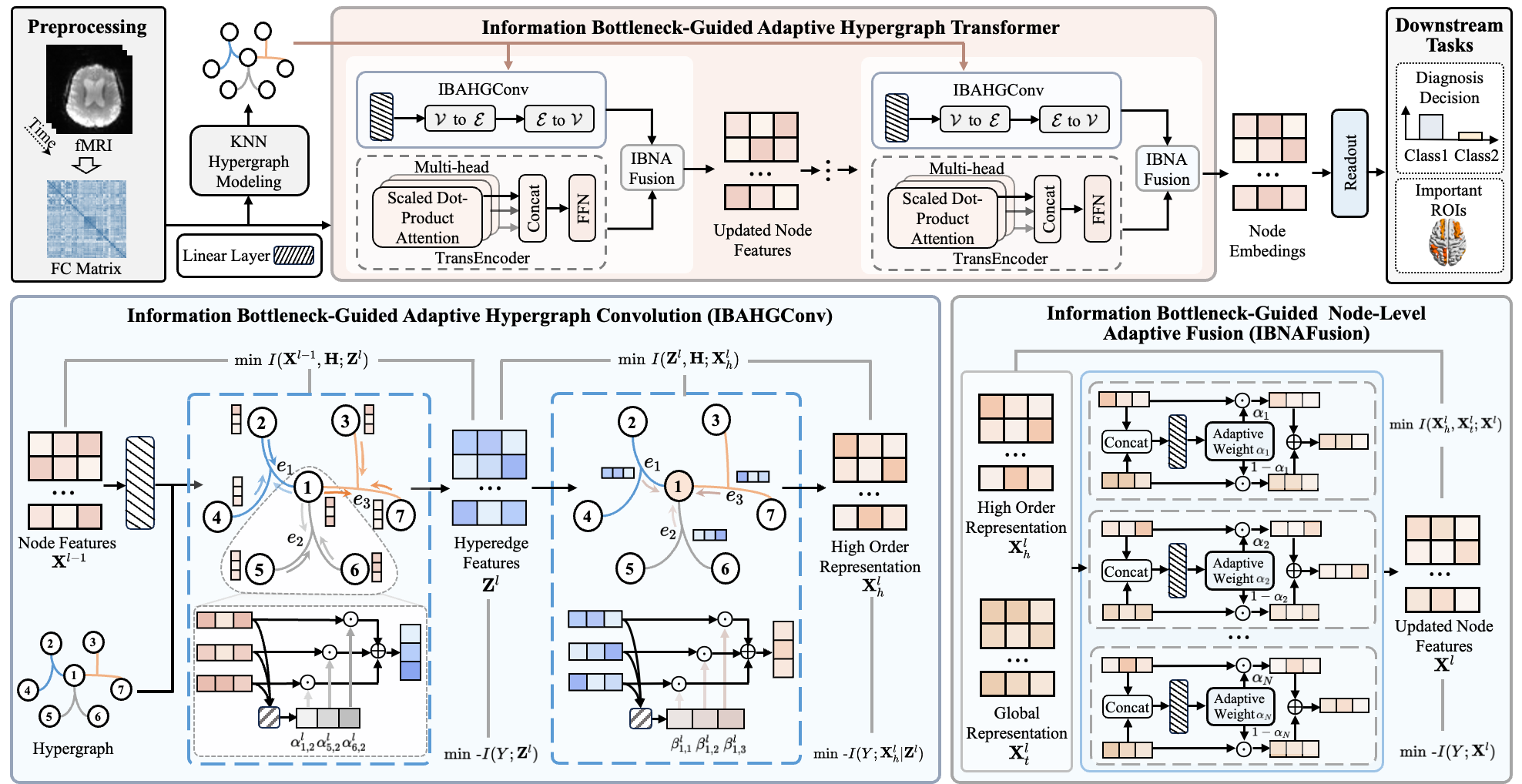}
\caption{The overall framework of the proposed IBAHGT.}
\label{fig:architecture}
\end{figure}
In brain network analysis, the ROIs are treated as nodes, and the functional connectivity (FC) matrix $\mathbf{X}^0 \in \mathbb{R}^{N \times N}$ is obtained by calculating Pearson correlation between ROIs, serving as the initial node features. 
Our goal is to comprehensively capture various types of correlations within brain networks and distill this multi-source information into efficient feature representations for predicting the target $Y$, (i.e., the ground truth class label).
The overall framework of IBAHGT is shown in Fig.\ref{fig:architecture}, with its core consisting of $L$ layers of information bottleneck-guided adaptive hypergraph Transformer layers, each comprising three main components:
information bottleneck-guided adaptive hypergraph convolution, an adaptive message passing process guided by the HIB principle, which produces high-order representation $\mathbf{X}_h^l$. Transformer encoder captures short- and long-range dependencies to obtain global representation $\mathbf{X}_t^l$. 
Information bottleneck-guided adaptive node-level fusion learns the fusion weights for each node, enabling fine-grained integration of MIMR high-order and global information to obtain updated node features $\mathbf{X}^l$. After $L$ layers, the final node embeddings $\mathbf{X}^L$ is flattened and fed into a multilayer perceptron (MLP) to predict class label.
\subsection{Information Bottleneck-Guided Adaptive Hypergraph Convolution}
Hypergraph-based methods can model high-order correlations among ROIs. However, general hypergraph convolution employs equal-weight message passing when capturing high-order information, which introduces substantial redundant information. To address this, the information bottleneck-guided adaptive hypergraph convolution adaptively learns weights during the message passing and introduces the HIB principle to optimize this process, capturing the MIMR high-order information.
\subsubsection{Adaptive Hypergraph Convolution}
The hypergraph is constructed using the K
Nearest Neighbor (KNN) method to model high-order correlations in the brain network. Specifically, ROIs are represented as a set of nodes $\mathcal{V} = \{v_1, v_2, \ldots, v_{N}\}$. The hypergraph structure is represented by an incidence matrix $\mathbf{H} \in \mathbb{R}^{N \times E}$, reflecting the node-hyperedge relation. The proximity between nodes is measured by the inverse of the Euclidean distance based on the node features $\mathbf{X}^0$.

Based on the constructed hypergraph structure, the adaptive hypergraph convolution performs a two-stage message passing process: adaptive gathering of node features to hyperedges and adaptive aggregating of hyperedge features to nodes.
In the first stage, each node adaptively learns the weight and performs weighted aggregation to obtain the MIMR hyperedge feature $\mathbf{Z}^l$, with the HIB principle introduced to optimize the weight learning process.
In the $l$-th layer, for a given hyperedge $e_j$, the process of generating the hyperedge feature $\mathbf{Z}_j^l$ based on the features of the nodes in its associated node set $\mathcal{N}_v(e_j)$ can be expressed as follows: 
\begin{equation}
    \mathbf{Z}_j^l = \sigma\left(\sum_{v_i \in \mathcal{N}_v(e_j)} \alpha_{i,j}^l \mathbf{W}_1^l \mathbf{X}_i^{l-1}\right), \quad \alpha_{i,j}^l = \frac{\exp{\left(\mathbf{a}_1^l \mathbf{W}_1^l \mathbf{X}_i^{l-1}\right)}}{\sum_{v_k \in \mathcal{N}_v(e_j)} \exp{\left(\mathbf{a}_1^l \mathbf{W}_1^l \mathbf{X}_k^{l-1}\right)}}\,,
\end{equation}
where $\mathbf{a}_1^l$ is a learnable vector, $\mathbf{W}^l_1$ is a learnable weight matrix, $\sigma$ represents the nonlinear activation function. $\alpha^l_{i,j}$ represents the weight when the feature of node $v_i$ is aggregated into the hyperedge $e_j$.

In the second stage, each hyperedge adaptively learns the weight for the information transfer back to the nodes and performs a weighted aggregation to obtain the high-order representation $\mathbf{X}_h^l$. This process is also optimized through the HIB principle.
For node $v_i$, the process of weighted aggregation of the hyperedge features from its connected hyperedge set $\mathcal{N}_e(v_i)$ to obtain the high-order representation $\mathbf{X}_{h,i}^l$ is formalized as follows:
\begin{equation}
    \mathbf{X}_{h,i}^l = \sigma\left(\sum_{e_j \in \mathcal{N}_e(v_i)} \beta_{i,j}^l \mathbf{W}_2^l \mathbf{Z}_j^l \right), \quad \beta_{i,j}^l = \frac{\exp{(\mathbf{a}_2^l \mathbf{W}_2^l \mathbf{Z}_j^l)}}{\sum_{e_p \in \mathcal{N}_e(v_i)} \exp{(\mathbf{a}_2^l \mathbf{W}_2^l \mathbf{Z}_p^l)}}\,,
\end{equation}

where $\mathbf{a}_2^l$ is a learnable vector, $\mathbf{W}^l_2$ is a learnable weight matrix, and $\beta^l_{ij}$ represents the weight when the feature of hyperedge $e_j$ is transmitted back to node $v_i$.
\subsubsection{Optimization Based on the HIB Principle}\label{sec:section3.3.2}
There is redundancy in the high-order information within the brain captured by adaptive hypergraph convolution, for this reason, we propose the HIB principle to optimize the information flow process to capture the MIMR high-order information.
In the $l$-th layer of adaptive hypergraph convolution, the optimization objective based on the HIB principle is:
\begin{equation}
    \min_{\mathbb{P}(\mathbf{Z}^{l}, \mathbf{X}^l \mid \mathbf{H}, \mathbf{X}^{l-1}) \in \Omega} \,\mbox{HIB}_{\gamma_1} (\mathbf{H}, \mathbf{X}^{l-1},Y; \mathbf{Z}^l,\mathbf{X}_h^{l})\triangleq -I(Y;\mathbf{Z}^l,\mathbf{X}_h^l)+\gamma_1I(\mathbf{X}^{l-1},\mathbf{H};\mathbf{Z}^l,\mathbf{X}_h^l)\,.
    \label{eq:hib}
\end{equation}

Minimizing $-I(Y;\mathbf{Z}^l,\mathbf{X}_h^l)$ ensures that the hypergraph message passing process focuses on task-relevant information, while minimizing $I(\mathbf{X}^{l-1},\mathbf{H};\mathbf{Z}^l,\mathbf{X}_h^l)$ effectively removes redundant information.
$\gamma_1$ is a hyperparameter that balances information and redundancy. However, accurately calculating mutual information terms in Eq.(\ref{eq:hib}) is intractable. 
Therefore, we introduce the lower bound of $I(Y;\mathbf{Z}^l,\mathbf{X}_h^l)$ and the upper bound of $I(\mathbf{X}_h^l,\mathbf{Z}^l;\mathbf{X}^{l-1},\mathbf{H})$ for subsequent estimation.

\begin{proposition}
\label{prop1}
  For any distribution $\mathbb{Q}_1(Y | \mathbf{Z}^{l})$ and $\mathbb{Q}_2(Y)$, we have
\begin{equation}
I(Y; \mathbf{Z}^{l},\mathbf{X}_h^l) = I(Y; \mathbf{Z}^{l})+I(Y;\mathbf{X}_h^l\mid\mathbf{Z}^{l})\,.
\label{cross}
\end{equation}
To represent the lower bounds of two terms on the right-hand side of the equation, in the adaptive hypergraph convolution layer, $\mathbf{Z}^l$ and $\mathbf{X}_h^l$ are read out and then mapped through the learnable parameter matrix to the predicted labels $\hat{y}^l_z$ and $\hat{y}^l_x$, respectively.
Therefore, the two terms on the right-hand side of the Eq.(\ref{cross}) can reduce to the cross-entropy loss without constants as follows:
\begin{equation}
    I(Y; \mathbf{Z}^{l}) \rightarrow -\mathcal{L}_{CE}(\hat{y}^l_z,Y) \quad , \quad I(Y; \mathbf{X}_h^{l}\mid\mathbf{Z}^{l}) \rightarrow -\mathcal{L}_{CE}(\hat{y}^l_x,Y)\,.
\end{equation}

The proof of the above process is given in Appendix \ref{app:appendixb}. Subsequently, we derive the upper bound of the second term in Eq.(\ref{eq:hib}).
\end{proposition}

\begin{proposition}
\label{prop2}
For any distributions $\mathbb{Q}(\mathbf{Z}^{l})$ and $\mathbb{Q}(\mathbf{X}_h^{l})$, we have
\begin{equation}
    I(\mathbf{X}^{l-1}, \mathbf{H}; \mathbf{X}_h^{l}, \mathbf{Z}^l) 
     \le I(\mathbf{X}^{l-1}, \mathbf{H}; \mathbf{Z}^l) + I(\mathbf{Z}^{l}, \mathbf{H};\mathbf{X}_h^l) \le \mbox{ZIB}^l+\mbox{XIB}^l\,,
\end{equation}
\begin{equation}
    \mbox{ZIB}^l = D_{KL} \left( \mathbb{P}(\mathbf{Z}^{l} \mid \mathbf{X}^{l-1},\mathbf{H}) \parallel \mathbb{Q}(\mathbf{Z}^{l}) \right) , \
    \mbox{XIB}^l = D_{KL} \left( \mathbb{P}(\mathbf{X}_h^{l} \mid \mathbf{Z}^{l},\mathbf{H}) \parallel 
    \mathbb{Q}(\mathbf{X}_h^{l}) \right)\,.
\end{equation}
The proof of the proposition can be
found in Appendix \ref{app:appendixb}. To estimate $\mbox{ZIB}^l$ and $\mbox{XIB}^l$, we set both $\mathbb{Q}(\mathbf{Z}^{l})$ and $\mathbb{Q}(\mathbf{X}_h^{l})$ as a mixture of Gaussians with learnable parameters \cite{dilokthanakul2016deep}. Concretely, set $\mathbb{Q}(\mathbf{Z}^l) \sim \sum_{i=1}^{m} w_i \, \mbox{Gaussian}(\mu_{0,i}, \sigma_{0,i}^2)$ and $\mathbb{Q}(\mathbf{X}_h^l) \sim \sum_{i=1}^{m} \hat{w}_i \, \mbox{Gaussian}(\hat{\mu}_{0,i}, \hat{\sigma}_{0,i}^2)$, where $w_i, \mu_{0,i}, \sigma_{0,i}$ and $\hat{w}_i, \hat{\mu}_{0,i}, \hat{\sigma}_{0,i}$ are learnable parameters.
Based on the aggregated $\mathbf{Z}^l$ and $\mathbf{X}^l$, the parameters are calculated and samples are drawn from a Gaussian distribution to obtain $\mathbb{P}(\mathbf{Z}^{l} \mid \mathbf{X}^{l-1},\mathbf{H}) \sim \mbox{Gaussian}(\mu_{l}, \sigma_{l}^2)$ and $\mathbb{P}(\mathbf{X}_h^{l} \mid \mathbf{Z}^{l},\mathbf{H}) \sim \mbox{Gaussian}(\hat{\mu}_{l}, \hat{\sigma}_{l}^2)$. The estimation of $\mbox{ZIB}^l$ and $\mbox{XIB}^l$ can be written as:
\begin{equation}
\widehat{\mbox{ZIB}}^l=\log ( \Phi( \mathbf{Z}^l; \mu_{l}, \sigma_{l}^2 ) ) 
   - \log ( \sum_{i=1}^m w_i\Phi( \mathbf{Z}^l; \mu_{0,i}, \sigma_{0,i}^2 ) )\,,
\end{equation}
\begin{equation}
\widehat{\mbox{XIB}}^l=\log ( \Phi( \mathbf{X}_h^l; \hat{\mu}_{l}, \hat{\sigma}_{l}^2 ) )-\log ( \sum_{i=1}^m \hat{w}_i\Phi( \mathbf{X}_h^l; \hat{\mu}_{0,i}, \hat{\sigma}_{0,i}^2 ) )\,.
\end{equation}

Plugging the estimations of both mutual information terms in Eq.(\ref{eq:hib}), the loss function for the information bottleneck-guided adaptive hypergraph convolution in the $l$-th layer can be obtained:
\begin{equation}
    \mathcal{L}^l_{H} = \mathcal{L}_{CE}(\hat{y}^l_z,Y)+ \mathcal{L}_{CE}(\hat{y}^l_x,Y)+
    \gamma_1[\widehat{\mbox{ZIB}}^l+\widehat{\mbox{XIB}}^l]\,.
\end{equation}
\end{proposition}

\subsection{Transformer Encoder}
Hypergraph-based learning methods struggle to capture long-range dependencies, whereas the self-attention mechanism in the Transformer encoder can effectively model both short- and long-range dependencies between ROIs in the brain network. 
Formally, in the $l$-th layer, a multi-head self-attention mechanism is used to learn global representation $\mathbf{X}_t^l$:
\begin{equation}
    \mathbf{X}_t^{l} = \left(\left\|_{m = 1}^M \mathbf{h}^{l,m}\right) \mathbf{W}_O^l \right), \quad \mathbf{h}^{l, m} = \text{softmax}\left(  \frac{\mathbf{W}_Q^{l,m} \mathbf{X}^{l - 1} \left( \mathbf{W}_K^{l,m} \mathbf{X}^{l - 1} \right)^\top}{\sqrt{d_K^{l,m}}}\right)\mathbf{W}_V^{l, m} \mathbf{X}^{l-1}\,,
\end{equation}

where $\Vert$ denotes the concatenation operator, $M$ is the number of heads, $\mathbf{W}_O^l$, $\mathbf{W}_Q^{l, m}$, $\mathbf{W}_K^{l, m}$, and $\mathbf{W}_V^{l, m}$ are learnable model parameters, and $d_K^{l, m}$ is the first dimension of $\mathbf{W}_K^{l, m}$. 
\subsection{Information Bottleneck-Guided Node-Level Adaptive Fusion}
Considering the varying roles of each ROI in the brain network and their different tendencies in capturing various dependencies \cite{shine2016dynamics}, the information bottleneck-guided node-level adaptive fusion learns an independent weight for each node to balance the high-order information from the information bottleneck-guided adaptive hypergraph convolution and the short- and long-range dependencies information from the Transformer encoder.
\subsubsection{Node-Level Adaptive Fusion}
In the $l$-th layer of the information bottleneck-guided adaptive hypergraph Transformer, node-level adaptive fusion integrates the high-order representation $\mathbf{X}_{h,i}^l$ and global representation $\mathbf{X}_{t,i}^l$ for each node $v_i$, resulting in the updated node features $\mathbf{X}_i^l$, which can be formulated as follows:
\begin{equation}
    \mathbf{X}_i^l = \theta_i^l\mathbf{X}_{h,i}^l+(1-\theta_i^l)\mathbf{X}_{t,i}^l\,,
\end{equation}
where $\theta_i^l$ denotes the learnable fusion weight of node $v_i$.

\subsubsection{Optimization Based on the IB Principle}

There is redundancy between the high-order representation of the brain network $\mathbf{X}_h^l$ and the global representation of the brain network $\mathbf{X}_t^l$ that contains short- and long-range dependencies. To reduce redundant information while preserving effective information in the updated node features $\mathbf{X}^l$, we apply the IB principle to optimize the fusion process. The optimization objective is:
\begin{equation}
    \min_{\mathbb{P}(\mathbf{X}^{l}, \mathbf{X}_h^{l},\mathbf{X}_t^{l}) \in \Omega} \,-I(Y;\mathbf{X}^l)+\gamma_2 I(\mathbf{X}_h^l, \mathbf{X}_t^l;\mathbf{X}^l)\,.
    \label{IB}
\end{equation}
Minimizing $-I(Y;\mathbf{X}^l)$ ensures that the fused features effectively retain task-relevant information, while minimizing $I(\mathbf{X}_h^l,\mathbf{X}_t^l;\mathbf{X}^{l})$  removes redundant information. The trade-off hyperparameter $\gamma_2$ is to balance the weights of two items.
In the $l$-th layer, $\mathbf{X}^l$ is read out and mapped to the predicted label $\hat{y}_f^l$ through a learnable parameter matrix.
Similar to Section \ref{sec:section3.3.2}, the minimization of the first term in Eq.(\ref{IB}) can be replaced by the cross-entropy loss, i.e.,
\begin{equation}
I(Y; \mathbf{X}^l) \rightarrow -\mathcal{L}_{CE}(\hat{y}^l_f,Y)\,.
\end{equation}
The second term in Eq.(\ref{IB}) has an optimizable upper bound, which is given as follows:

\begin{proposition}
\label{prop3}
For any distribution
$\mathbb{Q}(\mathbf{X}^l)$, we have
\begin{equation}
I(\mathbf{X}_h^l,\mathbf{X}_t^l;\mathbf{X}^l) \notag \le D_{KL} \left( \mathbb{P}(\mathbf{X}^{l} \mid \mathbf{X}_h^{l},\mathbf{X}_t^{l}) \parallel \mathbb{Q}(\mathbf{X}^{l})\right)\,.
\end{equation}
The proof of the proposition can be found in Appendix \ref{app:appendixb}. To specify the upper bound for $I(\mathbf{X}_h^l,\mathbf{X}_t^l;\mathbf{X}^l)$, we set $\mathbb{Q}(\mathbf{X}^l)$ as a mixture of Gaussians with learnable parameters: $\mathbb{Q}(\mathbf{X}^l) \sim \sum_{i=1}^{m} \tilde{w}_i \, \mbox{Gaussian}(\tilde{\mu}_{0,i}, \tilde{\sigma}_{0,i}^2)$. The parameters are computed based on the fused representation $\mathbf{X}^l$, and sampling is performed from a Gaussian distribution to obtain $\mathbb{P}(\mathbf{X}^{l} \mid \mathbf{X}_h^{l},\mathbf{X}_t^{l})\sim \text{Gaussian}(\tilde{\mu}_{l}, \tilde{\sigma}_{l}^2)$.
Thus, the estimation of $I(\mathbf{X}_h^l,\mathbf{X}_t^l;\mathbf{X}^l)$ is written as:
\begin{equation}  I(\mathbf{X}_h^l,\mathbf{X}_t^l;\mathbf{X}^l)\rightarrow \log ( \Phi( \mathbf{X}^l; \tilde{\mu}_{l}, \tilde{\sigma}_{l}^2 ) ) 
    - \log ( \sum_{i=1}^m \tilde{w}_i\Phi( \mathbf{X}^l; \tilde{\mu}_{0,i}, \tilde{\sigma}_{0,i}^2 ) )\,.
\end{equation}
By plugging the estimations of both mutual information terms into Eq.(\ref{IB}), the loss function for the node-level adaptive fusion process guided by IB principle in the $l$-th layer can be obtained:
\begin{equation}
    \mathcal{L}_{F}^l = \mathcal{L}_{CE}(\hat{y}^l_f,Y)+ \gamma_2(\log ( \Phi( \mathbf{X}^l; \tilde{\mu}_{l}, \tilde{\sigma}_{l}^2 ) ) 
    - \log ( \sum_{i=1}^m \tilde{w}_i\Phi( \mathbf{X}^l; \tilde{\mu}_{0,i}, \tilde{\sigma}_{0,i}^2 ) ))\,.
\end{equation}
\end{proposition}
\subsection{Information Bottleneck-Guided Adaptive Hypergraph Transformer Layer}
The information bottleneck-guided adaptive hypergraph Transformer layer is composed of parallel information bottleneck-guided adaptive hypergraph convolution (IBAHGConv) layer and Transformer encoder (TransEncoder) layer with information bottleneck-guided node-level adaptive fusion (IBNAFusion) module. In the $l$-th layer, the high order representation $\mathbf{X}_h^l$ from the $\mathrm{IBAHGConv}^l$ layer and the global representation $\mathbf{X}_t^l$ from the $\mathrm{TransEncoder}^l$ layer are adaptively aggregated by $\mathrm{IBNAFusion}^{l}$ to obtain updated node features $\mathbf{X}^l$, as shown in the following equation:
\begin{equation}
	\mathbf{X}_h^l = \mathrm{IBAHGConv}^l\left(\mathbf{X}^{l-1}, \mathbf{H}\right)\, ,
\end{equation}
\begin{equation}
	\mathbf{X}_t^l = \mathrm{TransEcoder}^l\left(\mathbf{X}^{l-1}\right)\, ,
\end{equation}
\begin{equation}
	\mathbf{X}^l = \mathrm{IBNAFusion}^{l}\left(\mathbf{X}_h^{l} , \mathbf{X}_t^l\right)\, .
\end{equation}
After $L$ layers, the final node embeddings $\mathbf{X}^L$ effectively capture information from various dependencies in the brain network.
\subsection{Loss Function}

By concatenating the node embeddings $\mathbf{X}^L$, the graph-level representation $\mathbf{\hat{x}}$ is obtained, which is then fed into an MLP to predict the label $\hat{y}$. During training, the overall loss function $\mathcal{L}$ is composed of the cross-entropy loss function $\mathcal{L}_{CE}$ for supervised disease prediction, the loss $\mathcal{L}_H$ for the adaptive hypergraph convolution part, and the loss $\mathcal{L}_{F}$ for the node-level adaptive fusion part:
\begin{equation}
    \mathcal{L} = \mathcal{L}_{CE}(\hat{y},Y)+ \sum_{l=1}^{L}( \lambda_1\mathcal{L}^l_{H}+\lambda_2\mathcal{L}^l_{F})\,,
\end{equation}
where $L$ denotes the number of layers in the information bottleneck-guided adaptive hypergraph Transformer, and $\lambda_1$ and $\lambda_2$ are trade-off hyperparameters balancing different losses.

\section{Experiments}
\subsection{Experiments Settings\label{sec:exp_setting}}
\noindent\textbf{Datasets and Preprocessing.}
The proposed method is evaluated on two brain network analysis related fMRI datasets, ABIDE \cite{heinsfeld2018identification} and ADNI \cite{jack2008alzheimer} . The ABIDE dataset includes brain imaging data from 871 subjects collected across 17 international sites, with 403 subjects diagnosed with Autism Spectrum Disorder (ASD). A subset of the ADNI dataset is used in this study, consisting of 66 patients with Alzheimer's Disease (AD) and 139 Normal Controls (NCs).

We preprocess the fMRI data using the standard steps provided by the Data Processing Assistant for Resting-State Function (DPARSF) toolkit \cite{song2011rest}, then use the Anatomical Automatic Labeling (AAL) brain template \cite{rolls2020automated} to divide the brain space into ROIs and extract the corresponding BOLD signals. The FC matrix is obtained by calculating the Pearson correlation coefficients between the ROIs.

\noindent\textbf{Metrics.}
Four machine learning and medical diagnostic-specific metrics are used: accuracy (ACC), sensitivity (SEN),  specificity (SPE) and area under the receiver operating characteristic curve (AUC).
We record the mean and standard deviation across 10 random runs on the test dataset.

\noindent\textbf{Implementation Details.}
In the proposed method, we set the value of $k$ to 3 in the KNN hypergraph modeling. The number of layers $L$ in the adaptive hypergraph Transformer and the number of heads $M$ in the multi-head attention module are set to 2 and 4, respectively. For all datasets, we randomly divide the training
set, evaluation set, and test set by the ratio of 7 : 1 : 2. During training, we use the Adam optimizer \cite{kingma2014adam}, with an initial learning rate set to $1 \times 10^{-4}$. The number of epochs is set to 300. All our experiments are implemented in PyTorch and trained on one NVIDIA 4090.

\subsection{Comparison Experiments}
\noindent\textbf{Baselines.}
The selected baselines correspond to three categories. The first category includes CNN-based method BrainNetCNN \cite{kawahara2017brainnetcnn}.  The second category includes graph or hypergraph-based learning methods, such as ContrastPool \cite{xu2024contrastive}, BrainGNN \cite{li2021braingnn}, BrainIB \cite{zheng2024brainib}, $\mbox{HGNN}^+$ \cite{gao2022hgnn+}, and FC-HAT \cite{ji2022fc}. The third category consists of Transformer-based methods, including BrainNetTF \cite{kan2022brain}, Com-BrainTF \cite{bannadabhavi2023community} and ALTER \cite{yu2024long}. The codes are reproduced based on the released codes.

\noindent\textbf{Results.}
Table \ref{table: compare} shows the comparison results between the proposed and baseline methods. On all datasets, IBAHGT significantly outperforms the baseline methods. The hypergraph-based methods achieved higher accuracy in most cases compared to graph-based methods, indicating that hypergraphs can capture high-order correlations that are beneficial for classification. Compared to these methods, our proposed approach further improves accuracy on the ABIDE dataset (ASD vs. NC: 77.6\%) and the ADNI dataset (AD vs. NC: 86.3\%). 
The performance improvement is attributed to our method, which captures the MIMR high-order correlations and long-range dependencies in the brain network, and condenses the information from both at the ROI level to obtain an efficient representation.

{\setlength{\tabcolsep}{2pt}
\begin{table*}[t]
\scriptsize
    \caption{Experimental results of the comparison methods.}
    \begin{minipage}{\linewidth}
        \centering   
        \resizebox{1\textwidth}{!}{
        \begin{tabular}{cllllllll}
        \toprule
        \multicolumn{1}{c}{{Datasets (Tasks)}} & \multicolumn{4}{c}{{ABIDE (ASD vs. NC)}} & \multicolumn{4}{c}{{ADNI (AD vs. NC)}} \\ 
        \cmidrule(lr){1-1} \cmidrule(lr){2-5} \cmidrule(lr){6-9}
        {Method} & {ACC (\%)} & {SPE (\%)} & {SEN (\%)} & {AUC (\%)} &  {ACC (\%)} & {SPE (\%)} & {SEN (\%)} & {AUC (\%)} \\ 
        \midrule
        BrainNetCNN & 66.2$\pm$1.9 & 74.3$\pm$4.1 & 56.8$\pm$3.1 & 68.3$\pm$2.6 & 73.2$\pm$2.7 & 79.3$\pm$4.2 & 60.0$\pm$5.8 & 66.3$\pm$4.7  \\
        
        ContrastPool & 63.0$\pm$2.5 & 71.1$\pm$8.3 & 53.5$\pm$4.6 & 68.1$\pm$7.7 & 70.7$\pm$4.1 & 72.1$\pm$5.7 & 67.7$\pm$10.2 & 69.5$\pm$3.9\\
        
        BrainGNN & 64.5$\pm$2.7 & 72.6$\pm$7.7 & 55.0$\pm$5.5 & 69.7$\pm$7.2 & 73.2$\pm$4.1 & 74.3$\pm$4.2 & 70.8$\pm$12.3 & 71.3$\pm$6.1  \\
        
        BrainIB& 68.3$\pm$1.9 & 76.4$\pm$5.4 & 58.8$\pm$5.0 & 70.2$\pm$1.9 & 74.1$\pm$2.5 & 78.6$\pm$5.1 & 64.6$\pm$11.5 & 69.0$\pm$6.5  \\
        
        $\mbox{HGNN}^+$& 68.9$\pm$2.1 & 77.0$\pm$4.4 & 59.3$\pm$5.9 & 71.0$\pm$1.9 & 74.6$\pm$2.9 & 81.4$\pm$4.2 & 60.0$\pm$5.8 & 69.2$\pm$6.8  \\
        
        FC-HAT & 70.7$\pm$3.5 & 75.5$\pm$5.0 & 65.0$\pm$6.5 & 72.0$\pm$4.8 & 77.1$\pm$3.7 & 84.3$\pm$3.6 & 61.5$\pm$4.9 & 69.1$\pm$2.0  \\
        
        BrainNetTF& 71.4$\pm$1.1 & 76.8$\pm$5.0 & 65.0$\pm$7.2 & 72.2$\pm$1.4 & 78.0$\pm$2.7 & 85.7$\pm$2.3 & 61.5$\pm$4.9 & 70.6$\pm$1.8  \\
        
        Com-BrainTF& 72.1$\pm$1.5 & 74.3$\pm$3.3 & 69.5$\pm$6.0 & 73.3$\pm$2.4 & 79.0$\pm$3.3 & 86.4$\pm$2.7 & 63.1$\pm$5.8 & 71.4$\pm$2.5  \\
        
        ALTER & 74.4$\pm$1.2 & 78.5$\pm$1.8 & 69.5$\pm$3.0 & 75.3$\pm$1.9& 81.5$\pm$2.5 & 87.9$\pm$1.8 & 67.7$\pm$5.8 & 71.9$\pm$2.8    \\
        
        IBAHGT & \textbf{77.6$\pm$0.8} & \textbf{80.4$\pm$3.1} & \textbf{74.3$\pm$3.8} & \textbf{80.5$\pm$2.0} & \textbf{86.3$\pm$1.2} & \textbf{88.6$\pm$4.7} & \textbf{81.5$\pm$7.9} & \textbf{82.8$\pm$5.2}  \\
        \bottomrule
        \end{tabular}}
    \end{minipage}
\label{table: compare}
\end{table*}}

\subsection{Ablation Study}
To validate the effectiveness of information bottleneck-guided adaptive hypergraph convolution (IBAHGConv) and information bottleneck-guided adaptive node-level fusion (IBNAFusion), ablation study are performed on two datasets, with results shown in Table \ref{table: ablation}. Using only IBAHGConv or the Transformer encoder (TransEncoder) performs worse than the combined method, indicating that integrating high-order information with long-range dependency information for comprehensive analysis enables a more thorough capture of abnormal changes in brain networks. On the ABIDE dataset, replacing IBAHGConv with $\mbox{HGNN}^+$ and AHGConv (without the HIB principle) led to a decrease in accuracy by 3.5\% and 2.5\%, respectively. This demonstrates that IBAHGConv optimizes the message-passing process and is more effective at capturing MIMR high-order information. Furthermore, diagnostic accuracy on the ABIDE dataset decreases when directly average fusion or NAFusion without IB principle, compared to IBNAFusion, which demonstrates the effectiveness of information concise during fusion at the node level guided by the IB principle.

\subsection{Interpretation Analysis}
\noindent\textbf{Discriminative Hyperedge Analysis}
We analyze the discriminative connectivity of brain networks in diseased and normal groups for each classification task, from a high-order correlation perspective.
By performing a t-test on the hyperedge features, we identify significant hyperedges, which are visualized in Fig.\ref{fig:3}.
For ASD vs. NC,
ASD patients show missing connections with right middle frontal gyrus (MFG.R) as well as right middle temporal gyrus (MTG.R), which foreshadows defects in individual external goal-related behavior and self-consciousness, further validating the conclusions of the previous study \cite{pankow2016aberrant}. For AD vs. NC, it can be observed that the high-order correlarions represented by the hyperedge in AD patients show missing connections with PHG.R and HIP.R. Previous studies have indicated that functional network impairments in AD patients often occur in the PHG and HIP, which affect memory \cite{miller2008age}.

{\setlength{\tabcolsep}{2pt}
\begin{table*}[t]
\scriptsize
    \caption{Experimental results of the ablation study.}
    \begin{minipage}{\linewidth}
        \centering   
        \resizebox{1\textwidth}{!}{
        \begin{tabular}{cllllllll}
        \toprule
        \multicolumn{1}{c}{{Datasets (Tasks)}} & \multicolumn{4}{c}{{ABIDE (ASD vs. NC)}} & \multicolumn{4}{c}{{ADNI (AD vs. NC)}} \\ 
        \cmidrule(lr){1-1} \cmidrule(lr){2-5} \cmidrule(lr){6-9}
        {Method} & {ACC (\%)} & {SPE (\%)} & {SEN (\%)} & {AUC (\%)} &  {ACC (\%)} & {SPE (\%)} & {SEN (\%)} & {AUC (\%)} \\ 
        \midrule
        TransEnoder &  69.3$\pm$0.9 & 77.9$\pm$4.5 & 59.3$\pm$4.2 & 72.7$\pm$5.3 & 76.1$\pm$2.8 & 77.9$\pm$6.9 & 72.3$\pm$10.4 & 75.6$\pm$5.8 \\

        IBAHGconv & 70.9$\pm$1.3 & 81.7$\pm$3.3 & 58.3$\pm$2.9 & 73.4$\pm$5.1& 78.0$\pm$2.2 & 82.9$\pm$5.3 & 67.7$\pm$10.2 & 77.3$\pm$3.8  \\

        IBAHGT w/o IBAHGConv & 74.1$\pm$1.4 & 76.6$\pm$4.9 & 71.3$\pm$6.4 & 77.8$\pm$3.4 &81.9$\pm$2.5 & 84.3$\pm$3.6 & 76.9$\pm$8.4 & 80.4$\pm$2.7 \\

        IBAHGT w/o HIB & 75.1$\pm$1.5 & 78.1$\pm$4.1 & 71.5$\pm$6.5 & 79.2$\pm$2.3 &83.4$\pm$2.4 & 84.3$\pm$3.6 & 81.5$\pm$7.9 & 81.0$\pm$5.1 \\
        
        IBAHGT w/o IBNAFusion & 73.0$\pm$1.3 & 77.5$\pm$4.6 & 67.8$\pm$6.4 & 77.4$\pm$3.5 & 81.0$\pm$1.8 & 85.7$\pm$3.9 & 70.8$\pm$5.8 & 80.3$\pm$3.9  \\
        
        IBAHGT w/o IB & 74.8$\pm$1.2 & 76.8$\pm$4.7 & 72.5$\pm$7.3 & 78.9$\pm$2.3 &83.9$\pm$2.0 & 85.7$\pm$3.2 & 80.0$\pm$6.2 & 81.8$\pm$4.6  \\

        IBAHGT & \textbf{77.6$\pm$0.8} & \textbf{80.4$\pm$3.1} & \textbf{74.3$\pm$3.8} & \textbf{80.5$\pm$2.0}& \textbf{86.3$\pm$1.2} & \textbf{88.6$\pm$4.7} & \textbf{81.5$\pm$7.9} & \textbf{82.8$\pm$5.2}  \\

        \bottomrule
        \end{tabular}}
    \end{minipage}
\label{table: ablation}
\end{table*}}
\begin{figure}[t]
\begin{center}
\includegraphics[width=0.9\linewidth]{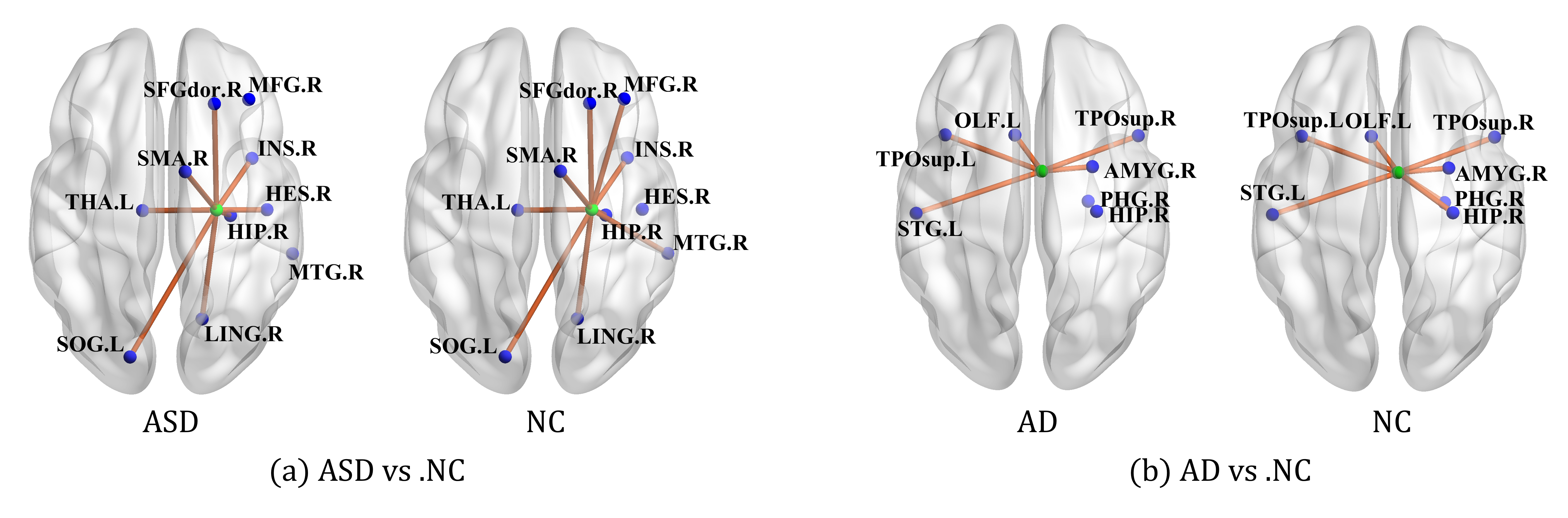}
\caption{Visualization of discriminative hyperedges for each classification. In each figure, the blue balls are the ROI contained in the hyperedge and the green balls are the center of the involved ROIs.}
\label{fig:3}
\end{center}
\end{figure}

\noindent\textbf{Analysis of Fusion Weights for Important ROIs}
To evaluate the influence of ROIs on the classification tasks, we measure the significance of the nodes using the t-test method and visualize the top 10 most important ROIs with the smallest p-values, displaying the fusion weights of each ROI through color differences. For ASD vs. NC, the identified important ROIs include the left precuneus (PCUN.L), right rectus (REC.R), right middle temporal gyrus (MTG.R), superior frontal gyrus, medial (SFGmed.L) and IPL.R, etc. As mentioned in \cite{kana2015aberrant}, the activations in PCUN and REC will be reduced when ASD patients inference mental states themselves or others. In addition, we found that the fusion weight of the MTG.R is larger, indicating a greater inclination towards high-order correlations. Previous studies have shown that MTG is critical for high-order cognitive functions, especially in language comprehension through collaboration with multiple other ROIs \cite{patterson2007you}. 
The fusion weight of the PCUN.L is relatively small, indicating that PCUN is more inclined to capture long-range dependencies. 
For AD vs. NC, the identified important ROIs include the HIP.R, left superior frontal gyrus, medial (SFGmed.L), right inferior frontal gyrus, triangular part (IFGtriang.R), right amygdala (AMYG.R), and left caudate nucleus (CAU.L), etc. 
The fusion weight of the HIP is relatively small, indicating that it relies more on long-range dependencies. Previous research has reported that low-frequency activity in the hippocampus can drive functional connectivity between the cortices of the two hemispheres \cite{chan2017low}, highlighting the HIP's key role in long-distance communication. 

\begin{figure}[t]
\includegraphics[width=\linewidth]{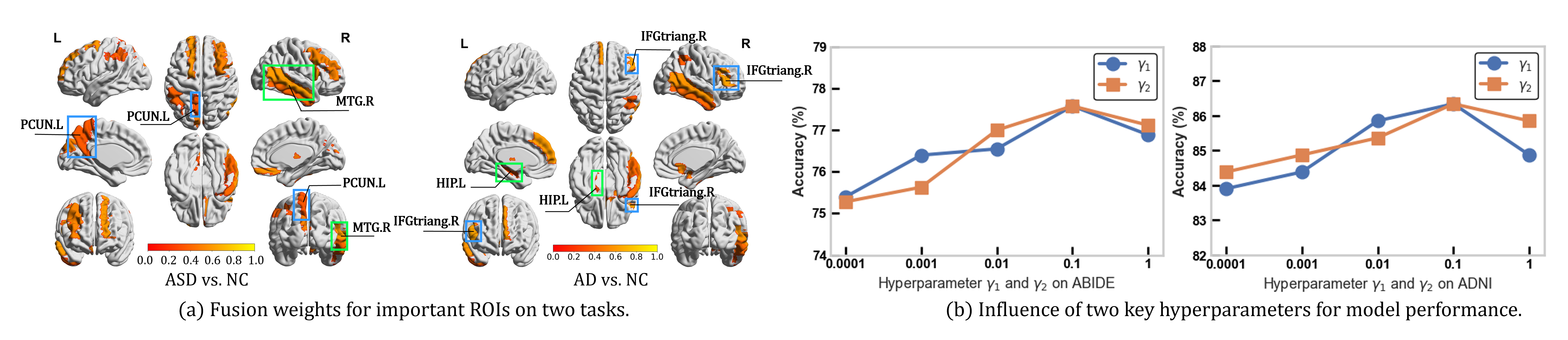}
\caption{Visualization of important ROIs with their fusion weights (Color bar represents the magnitude of the fusion weights) and the hyperparameters influence.}
\label{fig:4}
\end{figure}
\subsection{Hyperparameter Sensitivity Analysis}
To analyze how the trade-off hyperparameters $\gamma_1$ and $\gamma_2$ affect model performance, we perform a hyperparameter search and plot the optimal model performance on the ABIDE and ADNI datasets for each case where $\gamma_1 \in \{0.0001, 0.001, 0.01, 0.1, 1\}$ and $\gamma_2 \in \{0.0001, 0.001, 0.01, 0.1, 1\}$, as shown in Fig.\ref{fig:4}(b). We observe that as the magnitude of $\gamma_1$ increases, the accuracy on both datasets initially rises and then decreases, reaching a peak at $\gamma_1 = 0.1$. A similar trend is observed with $\gamma_2$, where the IBAHGT achieves the highest classification accuracy when $\gamma_2 = 0.1$.

\section{Conclusion}

In this paper, the proposed IBAHGT learns the MIMR high-order correlations and short- and long-range dependencies in brain networks within a unified framework. It suppresses redundancy while enabling fine-grained fusion of multi-source information to obtain efficient representations for brain disease diagnosis. The information bottleneck-guided adaptive hypergraph convolution introduces the HIB principle, adaptively learning message-passing weights to capture MIMR high-order information. 
Information bottleneck-guided node-level fusion assigns weights based on each node's information preference, effectively integrating high-order and global information. The proposed method achieves state-of-the-art performance on two datasets, offering new perspectives for understanding ROIs' interactions and brain information processing patterns.

\bibliographystyle{unsrt}
\bibliography{reference}

@article{seguin2023brain,
  title={Brain network communication: concepts, models and applications},
  author={Seguin, Caio and Sporns, Olaf and Zalesky, Andrew},
  journal={Nature reviews neuroscience},
  volume={24},
  number={9},
  pages={557--574},
  year={2023},
  publisher={Nature Publishing Group UK London}
}

@article{avena2018communication,
  title={Communication dynamics in complex brain networks},
  author={Avena-Koenigsberger, Andrea and Misic, Bratislav and Sporns, Olaf},
  journal={Nature reviews neuroscience},
  volume={19},
  number={1},
  pages={17--33},
  year={2018},
  publisher={Nature Publishing Group UK London}
}

@article{deco2021rare,
  title={Rare long-range cortical connections enhance human information processing},
  author={Deco, Gustavo and Perl, Yonathan Sanz and Vuust, Peter and Tagliazucchi, Enzo and Kennedy, Henry and Kringelbach, Morten L},
  journal={Current Biology},
  volume={31},
  number={20},
  pages={4436--4448},
  year={2021},
  publisher={Elsevier}
}

@article{vaswani2017attention,
  title={Attention is all you need},
  author={Vaswani, Ashish and Shazeer, Noam and Parmar, Niki and Uszkoreit, Jakob and Jones, Llion and Gomez, Aidan N and Kaiser, {\L}ukasz and Polosukhin, Illia},
  journal={Advances in neural information processing systems},
  volume={30},
  year={2017}
}

@article{bai2021hypergraph,
  title={Hypergraph convolution and hypergraph attention},
  author={Bai, Song and Zhang, Feihu and Torr, Philip HS},
  journal={Pattern Recognition},
  volume={110},
  pages={107637},
  year={2021},
  publisher={Elsevier}
}

@article{yang2023mapping,
  title={Mapping multi-modal brain connectome for brain disorder diagnosis via cross-modal mutual learning},
  author={Yang, Yanwu and Ye, Chenfei and Guo, Xutao and Wu, Tao and Xiang, Yang and Ma, Ting},
  journal={IEEE Transactions on Medical Imaging},
  volume={43},
  number={1},
  pages={108--121},
  year={2023},
  publisher={IEEE}
}

@article{ye2023rh,
  title={RH-BrainFS: regional heterogeneous multimodal brain networks fusion strategy},
  author={Ye, Hongting and Zheng, Yalu and Li, Yueying and Zhang, Ke and Kong, Youyong and Yuan, Yonggui},
  journal={Advances in Neural Information Processing Systems},
  volume={36},
  pages={59286--59303},
  year={2023}
}

@article{wang2023dynamic,
  title={Dynamic weighted hypergraph convolutional network for brain functional connectome analysis},
  author={Wang, Junqi and Li, Hailong and Qu, Gang and Cecil, Kim M and Dillman, Jonathan R and Parikh, Nehal A and He, Lili},
  journal={Medical image analysis},
  volume={87},
  pages={102828},
  year={2023},
  publisher={Elsevier}
}

@article{tishby2000information,
  title={The information bottleneck method},
  author={Tishby, Naftali and Pereira, Fernando C and Bialek, William},
  journal={arXiv preprint physics/0004057},
  year={2000}
}

@inproceedings{luo2019significance,
  title={Significance-aware information bottleneck for domain adaptive semantic segmentation},
  author={Luo, Yawei and Liu, Ping and Guan, Tao and Yu, Junqing and Yang, Yi},
  booktitle={Proceedings of the IEEE/CVF international conference on computer vision},
  pages={6778--6787},
  year={2019}
}

@article{peng2018variational,
  title={Variational discriminator bottleneck: Improving imitation learning, inverse rl, and gans by constraining information flow},
  author={Peng, Xue Bin and Kanazawa, Angjoo and Toyer, Sam and Abbeel, Pieter and Levine, Sergey},
  journal={arXiv preprint arXiv:1810.00821},
  year={2018}
}

@inproceedings{wang2020learning,
  title={Learning efficient multi-agent communication: An information bottleneck approach},
  author={Wang, Rundong and He, Xu and Yu, Runsheng and Qiu, Wei and An, Bo and Rabinovich, Zinovi},
  booktitle={International conference on machine learning},
  pages={9908--9918},
  year={2020},
  organization={PMLR}
}

@article{dilokthanakul2016deep,
  title={Deep unsupervised clustering with gaussian mixture variational autoencoders},
  author={Dilokthanakul, Nat and Mediano, Pedro AM and Garnelo, Marta and Lee, Matthew CH and Salimbeni, Hugh and Arulkumaran, Kai and Shanahan, Murray},
  journal={arXiv preprint arXiv:1611.02648},
  year={2016}
}

@article{shine2016dynamics,
  title={The dynamics of functional brain networks: integrated network states during cognitive task performance},
  author={Shine, James M and Bissett, Patrick G and Bell, Peter T and Koyejo, Oluwasanmi and Balsters, Joshua H and Gorgolewski, Krzysztof J and Moodie, Craig A and Poldrack, Russell A},
  journal={Neuron},
  volume={92},
  number={2},
  pages={544--554},
  year={2016},
  publisher={Elsevier}
}

@article{liu2017complex,
  title={Complex brain network analysis and its applications to brain disorders: a survey},
  author={Liu, Jin and Li, Min and Pan, Yi and Lan, Wei and Zheng, Ruiqing and Wu, Fang-Xiang and Wang, Jianxin},
  journal={Complexity},
  volume={2017},
  number={1},
  pages={8362741},
  year={2017},
  publisher={Wiley Online Library}
}

@article{supekar2008network,
  title={Network analysis of intrinsic functional brain connectivity in Alzheimer's disease},
  author={Supekar, Kaustubh and Menon, Vinod and Rubin, Daniel and Musen, Mark and Greicius, Michael D},
  journal={PLoS computational biology},
  volume={4},
  number={6},
  pages={e1000100},
  year={2008},
  publisher={Public Library of Science San Francisco, USA}
}

@article{benson2016higher,
  title={Higher-order organization of complex networks},
  author={Benson, Austin R and Gleich, David F and Leskovec, Jure},
  journal={Science},
  volume={353},
  number={6295},
  pages={163--166},
  year={2016},
  publisher={American Association for the Advancement of Science}
}

@article{dajani2016local,
  title={Local brain connectivity across development in autism spectrum disorder: A cross-sectional investigation},
  author={Dajani, Dina R and Uddin, Lucina Q},
  journal={Autism Research},
  volume={9},
  number={1},
  pages={43--54},
  year={2016},
  publisher={Wiley Online Library}
}

@article{park2013structural,
  title={Structural and functional brain networks: from connections to cognition},
  author={Park, Hae-Jeong and Friston, Karl},
  journal={Science},
  volume={342},
  number={6158},
  pages={1238411},
  year={2013},
  publisher={American Association for the Advancement of Science}
}

@article{barttfeld2011big,
  title={A big-world network in ASD: dynamical connectivity analysis reflects a deficit in long-range connections and an excess of short-range connections},
  author={Barttfeld, Pablo and Wicker, Bruno and Cukier, Sebasti{\'a}n and Navarta, Silvana and Lew, Sergio and Sigman, Mariano},
  journal={Neuropsychologia},
  volume={49},
  number={2},
  pages={254--263},
  year={2011},
  publisher={Elsevier}
}

@incollection{von1994correlation,
  title={The correlation theory of brain function},
  author={Von Der Malsburg, Christoph},
  booktitle={Models of neural networks: Temporal aspects of coding and information processing in biological systems},
  pages={95--119},
  year={1994},
  publisher={Springer}
}

@article{luna2004emergence,
  title={The emergence of collaborative brain function: FMRI studies of the development of response inhibition},
  author={Luna, Beatriz and Sweeney, John A},
  journal={Annals of the New York Academy of Sciences},
  volume={1021},
  number={1},
  pages={296--309},
  year={2004},
  publisher={Wiley Online Library}
}

@article{bessadok2022graph,
  title={Graph neural networks in network neuroscience},
  author={Bessadok, Alaa and Mahjoub, Mohamed Ali and Rekik, Islem},
  journal={IEEE Transactions on Pattern Analysis and Machine Intelligence},
  volume={45},
  number={5},
  pages={5833--5848},
  year={2022},
  publisher={IEEE}
}

@article{qu2021brain,
  title={Brain functional connectivity analysis via graphical deep learning},
  author={Qu, Gang and Hu, Wenxing and Xiao, Li and Wang, Junqi and Bai, Yuntong and Patel, Beenish and Zhang, Kun and Wang, Yu-Ping},
  journal={IEEE Transactions on Biomedical Engineering},
  volume={69},
  number={5},
  pages={1696--1706},
  year={2021},
  publisher={IEEE}
}

@article{niu2023applications,
  title={Applications of hypergraph-based methods in classifying and subtyping psychiatric disorders: a survey},
  author={Niu, Ju and Du, Yuhui},
  journal={Radiology Science},
  volume={2},
  number={01},
  pages={83--95},
  year={2023},
  publisher={Compuscript Ltd. Bay 11a, Shannon Industrial Est., Shannon, Co. Clare, Ireland}
}

@inproceedings{feng2019hypergraph,
  title={Hypergraph neural networks},
  author={Feng, Yifan and You, Haoxuan and Zhang, Zizhao and Ji, Rongrong and Gao, Yue},
  booktitle={Proceedings of the AAAI conference on artificial intelligence},
  volume={33},
  number={01},
  pages={3558--3565},
  year={2019}
}

@article{kawahara2017brainnetcnn,
  title={BrainNetCNN: Convolutional neural networks for brain networks; towards predicting neurodevelopment},
  author={Kawahara, Jeremy and Brown, Colin J and Miller, Steven P and Booth, Brian G and Chau, Vann and Grunau, Ruth E and Zwicker, Jill G and Hamarneh, Ghassan},
  journal={NeuroImage},
  volume={146},
  pages={1038--1049},
  year={2017},
  publisher={Elsevier}
}

@article{zheng2024brainib,
  title={Brainib: Interpretable brain network-based psychiatric diagnosis with graph information bottleneck},
  author={Zheng, Kaizhong and Yu, Shujian and Li, Baojuan and Jenssen, Robert and Chen, Badong},
  journal={IEEE Transactions on Neural Networks and Learning Systems},
  year={2024},
  publisher={IEEE}
}

@article{ji2022fc,
  title={FC--HAT: Hypergraph attention network for functional brain network classification},
  author={Ji, Junzhong and Ren, Yating and Lei, Minglong},
  journal={Information Sciences},
  volume={608},
  pages={1301--1316},
  year={2022},
  publisher={Elsevier}
}

@article{kan2022brain,
  title={Brain network transformer},
  author={Kan, Xuan and Dai, Wei and Cui, Hejie and Zhang, Zilong and Guo, Ying and Yang, Carl},
  journal={Advances in Neural Information Processing Systems},
  volume={35},
  pages={25586--25599},
  year={2022}
}

@inproceedings{bannadabhavi2023community,
  title={Community-aware transformer for autism prediction in fmri connectome},
  author={Bannadabhavi, Anushree and Lee, Soojin and Deng, Wenlong and Ying, Rex and Li, Xiaoxiao},
  booktitle={International Conference on Medical Image Computing and Computer-Assisted Intervention},
  pages={287--297},
  year={2023},
  organization={Springer}
}

@article{gao2022hgnn+,
  title={Hgnn+: General hypergraph neural networks},
  author={Gao, Yue and Feng, Yifan and Ji, Shuyi and Ji, Rongrong},
  journal={IEEE Transactions on Pattern Analysis and Machine Intelligence},
  volume={45},
  number={3},
  pages={3181--3199},
  year={2022},
  publisher={IEEE}
}

@article{xu2024contrastive,
  title={Contrastive graph pooling for explainable classification of brain networks},
  author={Xu, Jiaxing and Bian, Qingtian and Li, Xinhang and Zhang, Aihu and Ke, Yiping and Qiao, Miao and Zhang, Wei and Sim, Wei Khang Jeremy and Guly{\'a}s, Bal{\'a}zs},
  journal={IEEE Transactions on Medical Imaging},
  year={2024},
  publisher={IEEE}
}

@article{li2021braingnn,
  title={Braingnn: Interpretable brain graph neural network for fmri analysis},
  author={Li, Xiaoxiao and Zhou, Yuan and Dvornek, Nicha and Zhang, Muhan and Gao, Siyuan and Zhuang, Juntang and Scheinost, Dustin and Staib, Lawrence H and Ventola, Pamela and Duncan, James S},
  journal={Medical Image Analysis},
  volume={74},
  pages={102233},
  year={2021},
  publisher={Elsevier}
}

@article{yu2024long,
  title={Long-range brain graph transformer},
  author={Yu, Shuo and Jin, Shan and Li, Ming and Sarwar, Tabinda and Xia, Feng},
  journal={Advances in Neural Information Processing Systems},
  volume={37},
  pages={24472--24495},
  year={2024}
}

@article{heinsfeld2018identification,
  title={Identification of autism spectrum disorder using deep learning and the ABIDE dataset},
  author={Heinsfeld, Anibal S{\'o}lon and Franco, Alexandre Rosa and Craddock, R Cameron and Buchweitz, Augusto and Meneguzzi, Felipe},
  journal={NeuroImage: Clinical},
  volume={17},
  pages={16--23},
  year={2018},
  publisher={Elsevier}
}

@article{jack2008alzheimer,
  title={The Alzheimer's disease neuroimaging initiative (ADNI): MRI methods},
  author={Jack Jr, Clifford R and Bernstein, Matt A and Fox, Nick C and Thompson, Paul and Alexander, Gene and Harvey, Danielle and Borowski, Bret and Britson, Paula J and L. Whitwell, Jennifer and Ward, Chadwick and others},
  journal={Journal of Magnetic Resonance Imaging: An Official Journal of the International Society for Magnetic Resonance in Medicine},
  volume={27},
  number={4},
  pages={685--691},
  year={2008},
  publisher={Wiley Online Library}
}

@article{song2011rest,
  title={REST: a toolkit for resting-state functional magnetic resonance imaging data processing},
  author={Song, Xiao-Wei and Dong, Zhang-Ye and Long, Xiang-Yu and Li, Su-Fang and Zuo, Xi-Nian and Zhu, Chao-Zhe and He, Yong and Yan, Chao-Gan and Zang, Yu-Feng},
  journal={PloS one},
  volume={6},
  number={9},
  pages={e25031},
  year={2011},
  publisher={Public Library of Science San Francisco, USA}
}

@article{rolls2020automated,
  title={Automated anatomical labelling atlas 3},
  author={Rolls, Edmund T and Huang, Chu-Chung and Lin, Ching-Po and Feng, Jianfeng and Joliot, Marc},
  journal={Neuroimage},
  volume={206},
  pages={116189},
  year={2020},
  publisher={Elsevier}
}

@article{kingma2014adam,
  title={Adam: A method for stochastic optimization},
  author={Kingma, Diederik P and Ba, Jimmy},
  journal={arXiv preprint arXiv:1412.6980},
  year={2014}
}

@article{agam2010reduced,
  title={Reduced cognitive control of response inhibition by the anterior cingulate cortex in autism spectrum disorders},
  author={Agam, Yigal and Joseph, Robert M and Barton, Jason JS and Manoach, Dara S},
  journal={Neuroimage},
  volume={52},
  number={1},
  pages={336--347},
  year={2010},
  publisher={Elsevier}
}

@article{pankow2016aberrant,
  title={Aberrant salience is related to dysfunctional self-referential processing in psychosis},
  author={Pankow, Anne and Katthagen, Teresa and Diner, Sarah and Deserno, Lorenz and Boehme, Rebecca and Kathmann, Nobert and Gleich, Tobias and Gaebler, Michael and Walter, Henrik and Heinz, Andreas and others},
  journal={Schizophrenia bulletin},
  volume={42},
  number={1},
  pages={67--76},
  year={2016},
  publisher={Oxford University Press US}
}

@article{kana2015aberrant,
  title={Aberrant functioning of the theory-of-mind network in children and adolescents with autism},
  author={Kana, Rajesh K and Maximo, Jose O and Williams, Diane L and Keller, Timothy A and Schipul, Sarah E and Cherkassky, Vladimir L and Minshew, Nancy J and Just, Marcel Adam},
  journal={Molecular autism},
  volume={6},
  pages={1--12},
  year={2015},
  publisher={Springer}
}

@article{patterson2007you,
  title={Where do you know what you know? The representation of semantic knowledge in the human brain},
  author={Patterson, Karalyn and Nestor, Peter J and Rogers, Timothy T},
  journal={Nature reviews neuroscience},
  volume={8},
  number={12},
  pages={976--987},
  year={2007},
  publisher={Nature Publishing Group UK London}
}

@article{chan2017low,
  title={Low-frequency hippocampal--cortical activity drives brain-wide resting-state functional MRI connectivity},
  author={Chan, Russell W and Leong, Alex TL and Ho, Leon C and Gao, Patrick P and Wong, Eddie C and Dong, Celia M and Wang, Xunda and He, Jufang and Chan, Ying-Shing and Lim, Lee Wei and others},
  journal={Proceedings of the National Academy of sciences},
  volume={114},
  number={33},
  pages={E6972--E6981},
  year={2017},
  publisher={National Academy of Sciences}
}

@article{myers2014within,
  title={Within-patient correspondence of amyloid-$\beta$ and intrinsic network connectivity in Alzheimer’s disease},
  author={Myers, Nicholas and Pasquini, Lorenzo and G{\"o}ttler, Jens and Grimmer, Timo and Koch, Kathrin and Ortner, Marion and Neitzel, Julia and M{\"u}hlau, Mark and F{\"o}rster, Stefan and Kurz, Alexander and others},
  journal={Brain},
  volume={137},
  number={7},
  pages={2052--2064},
  year={2014},
  publisher={Oxford University Press}
}

@article{ilioska2023connectome,
  title={Connectome-wide mega-analysis reveals robust patterns of atypical functional connectivity in autism},
  author={Ilioska, Iva and Oldehinkel, Marianne and Llera, Alberto and Chopra, Sidhant and Looden, Tristan and Chauvin, Roselyne and Van Rooij, Daan and Floris, Dorothea L and Tillmann, Julian and Moessnang, Carolin and others},
  journal={Biological psychiatry},
  volume={94},
  number={1},
  pages={29--39},
  year={2023},
  publisher={Elsevier}
}

@article{marco2011sensory,
  title={Sensory processing in autism: a review of neurophysiologic findings},
  author={Marco, Elysa J and Hinkley, Leighton BN and Hill, Susanna S and Nagarajan, Srikantan S},
  journal={Pediatric research},
  volume={69},
  number={8},
  pages={48--54},
  year={2011},
  publisher={Nature Publishing Group}
}

@article{miller2008age,
  title={Age-related memory impairment associated with loss of parietal deactivation but preserved hippocampal activation},
  author={Miller, Saul L and Celone, Kim and DePeau, Kristina and Diamond, Eli and Dickerson, Bradford C and Rentz, Dorene and Pihlajam{\"a}ki, Maija and Sperling, Reisa A},
  journal={Proceedings of the National Academy of Sciences},
  volume={105},
  number={6},
  pages={2181--2186},
  year={2008},
  publisher={National Academy of Sciences}
}

@article{zhao2018evaluating,
  title={Evaluating functional connectivity of executive control network and frontoparietal network in Alzheimer’s disease},
  author={Zhao, Qinghua and Lu, Hong and Metmer, Hichem and Li, Will XY and Lu, Jianfeng},
  journal={Brain research},
  volume={1678},
  pages={262--272},
  year={2018},
  publisher={Elsevier}
}

@article{adhd2012adhd,
  title={The ADHD-200 consortium: a model to advance the translational potential of neuroimaging in clinical neuroscience},
  author={ADHD-200 consortium},
  journal={Frontiers in systems neuroscience},
  volume={6},
  pages={62},
  year={2012},
  publisher={Frontiers Research Foundation}
}
\newpage
\appendix
\section{Algorithm}

The IBAHGT algorithm is described in Algorithm 1. To clearly present the process of our algorithm, the IBAHGT algorithm is described in the form of Algorithm 1. In the implementation process, we express the information-bottleneck-guided adaptive hypergraph convolution and the information-bottleneck-guided node-level adaptive fusion in matrix form to accelerate computation.
\begin{algorithm}[]
    \SetAlgoLined
    \SetKwInOut{Input}{Input}
    \SetKwInOut{Output}{Output}
    \Input{Initial node features $\mathbf{X}^0$; Incidence matrix $\mathbf{H}$}
    \SetKwInOut{Initial}{Initialize}
    \Initial{
        The learnable weight matrix $\mathbf{W}_1^l$, $\mathbf{W}_2^l$; The learnable vector $\mathbf{a}_1^l$, $\mathbf{a}_2^l$, $\theta^l$ \\
    }
    \Output{The predicted class label $\hat{y}$}

    \For{Layers $l = 1,2,\dots,L$}{
        \SetKwBlock{StageOne}{Stage 1: Information-Bottleneck-Guided Adaptive Hypergraph Convolution}{}
        \StageOne{
            \SetKwBlock{PhaseOne}{Phase 1:Adaptive Gathering of Node Features to Hyperedges}{}
            \PhaseOne{
                \For{$e_j \in \mathcal{E}$}{
                    \For{$v_i \in \mathcal{V}$}{
                        Construct Node Set $\mathcal{N}_v(e_j) = \{v_i\in\mathcal{V} \mid \mathbf{H}_{i,j} \neq 0\}$; \\
                        $\alpha_{i,j}^l = \frac{\exp{\left(\mathbf{a}_1^l \mathbf{W}_1^l \mathbf{X}_i^{l-1}\right)}}{\sum_{v_k \in \mathcal{N}_v(e_j)} \exp{\left(\mathbf{a}_1^l \mathbf{W}_1^l \mathbf{X}_k^{l-1}\right)}}$; \\
                    }
                    $\mathbf{Z}_j^l = \sigma\left(\sum_{v_i \in \mathcal{N}_v(e_j)} \alpha_{i,j}^l \mathbf{W}_1^l \mathbf{X}_i^{l-1}\right)$;\\
                }
                $\mu_{l} \leftarrow \mathbf{Z}^l[0:f_z^l]$;\\
                $\sigma^2_{l} \leftarrow \mbox{softplus}(\mathbf{Z}^l[f_z^l:2f_z^l])$;\\
                $ \mathbb{P}(\mathbf{Z}^l \mid \mathbf{X}^{l-1},\mathbf{H}) \sim \mbox{Gaussian}(\mu_{l}, \sigma_{l}^2)$;\\
            }
            \SetKwBlock{PhaseTwo}{Phase 2:Adaptive Aggregating of Hyperedge Features to Nodes}{}
            \PhaseTwo{
                \For{$v_i \in \mathcal{V}$}{
                    \For{$e_j \in \mathcal{E}$}{
                        Construct Hyperedge Set $\mathcal{N}_e(v_i) = \{e_j\in\mathcal{E} \mid \mathbf{H}_{i,j} \neq 0\}$; \\
                        $\beta_{i,j}^l = \frac{\exp{(\mathbf{a}_2^l \mathbf{W}_2^l \mathbf{Z}_j^l)}}{\sum_{e_p \in \mathcal{N}_e(v_i)} \exp{(\mathbf{a}_2^l \mathbf{W}_2^l \mathbf{Z}_p^l)}}$;\\
                    }
                    $\mathbf{X}_{h,i}^l = \sigma\left(\sum_{e_j \in \mathcal{N}_e(v_i)} \beta_{i,j}^l \mathbf{W}_2^l \mathbf{Z}_j^l \right)$;\\
                }
                $\hat{\mu}_{l} \leftarrow \mathbf{X}_h^l[0:f_h^l]$;\\
                $\hat{\sigma}^2_{l} \leftarrow \mbox{softplus}(\mathbf{X}_h^l[f_h^l:2f_h^l])$;\\
                $\mathbb{P}(\mathbf{X}_h^l \mid \mathbf{Z}^{l},\mathbf{H}) \sim \mbox{Gaussian}(\hat{\mu}_{l}, \hat{\sigma}_{l}^2)$;\\
            }
        }
        \SetKwBlock{StageTwo}{Stage 2: Transformer Encoder}{}
        \StageTwo{
            $\mathbf{X}_t^l \leftarrow \mbox{TransEncoder}(\mathbf{X}^{l-1})$
        }
        \SetKwBlock{StageThree}{Stage 3: Information-Bottleneck-Guided Node-Level Adaptive Fusion}{}
        \StageThree{
            \For{$v_i \in \mathcal{V}$}{
                $\mathbf{X}_i^l = \theta_i^l\mathbf{X}_{h,i}^l+(1-\theta_i^l)\mathbf{X}_{t,i}^l$;\\
            }
            $\tilde{\mu}_l \leftarrow \mathbf{X}^l[0:f^l]$;\\
            $\tilde{\sigma}^2_l \leftarrow \mbox{softplus}(\mathbf{X}^l[f^l:2f^l])$;\\
            $\mathbb{P}(\mathbf{X}^l \mid \mathbf{X}_h^l,\mathbf{X}_t^l) \sim \mbox{Gaussian}(\tilde{\mu}_l, \tilde{\sigma}_l^2)$;\\
        }
    }
    $\hat{\mathbf{x}} = \mbox{Readout}(\mathbf{X}^L))$; \\
    $\hat{y} = \mbox{MLP}(\hat{\mathbf{x}})$;\\
    \Return{$\hat{y}$};\\
    \caption{Information-Bottleneck-Guided Adaptive Hypergraph Transformer}
\end{algorithm}

\section{Proof\label{app:appendixb}}
\begin{figure}[t]
\includegraphics[width=\textwidth]{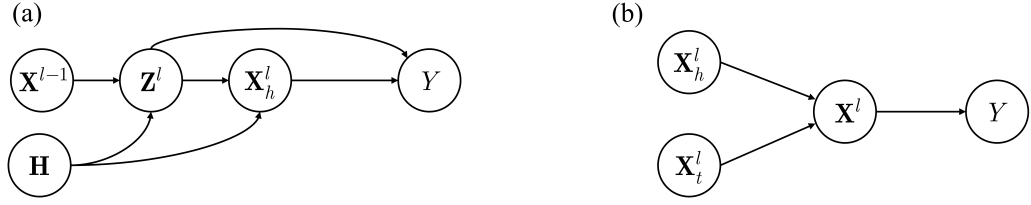}
\caption{(a) The Markovian chain of information bottleneck-guided adaptive hypergraph convolution. (b) The Markovian chain of information bottleneck-guided node-level adaptive fusion. }
\label{fig:markov_chain}
\end{figure}

\noindent\textbf{Proof for Proposition \ref{prop1}}
We restate Proposition \ref{prop1}:
\begin{align}
I(Y; \mathbf{Z}^{l},\mathbf{X}_h^l) = I(Y; \mathbf{Z}^{l})+I(Y;\mathbf{X}_h^l\mid \mathbf{Z}^{l})\,.
\end{align}
For any probabilistic distribution functions $\mathbb{Q}_1(Y \mid \mathbf{Z}^{l})$, $\mathbb{Q}_1(Y \mid \mathbf{Z}^{l}, \mathbf{X}_h^l)$, and $\mathbb{Q}_2(Y)$, we can express $I(Y; \mathbf{Z}^{l})$ and $I(Y; \mathbf{X}_h^l \mid \mathbf{Z}^{l})$ as follows:
\begin{equation}
I(Y; \mathbf{Z}^{l})\geq 1 + \mathbb{E} \left[ \log \frac{ \mathbb{Q}_1(Y \mid \mathbf{Z}^{l})}{\mathbb{Q}_2(Y)} \right] + \mathbb{E}_{\mathbb{P}(Y) \mathbb{P}(\mathbf{Z}^{l})}\left[  \frac{\mathbb{Q}_1(Y \mid \mathbf{Z}^{l})}{\mathbb{Q}_2(Y)} \right]\,,
\end{equation}

\begin{equation}
I(Y; \mathbf{X}_h^{l}\mid \mathbf{Z}^l)\geq 1 + \mathbb{E}_{\mathbb{P}(Y,\mathbf{X}_h^l \mid \mathbf{Z}^l)} \left[ \log \frac{ \mathbb{Q}_1(Y \mid \mathbf{Z}^{l},\mathbf{X}_h^l)}{\mathbb{Q}_2(Y)} \right] + \mathbb{E}_{\mathbb{P}(Y) \mathbb{P}(\mathbf{X}_h^{l}\mid\mathbf{Z}^l)}\left[ \frac{\mathbb{Q}_1(Y \mid \mathbf{Z}^{l}, \mathbf{X}_h^l)}{\mathbb{Q}_2(Y)} \right]
\end{equation}
The Nguyen, Wainright \& Jordan's bound is used here:

\noindent\textbf{Lemma 1}
For any two random variables $\mathbf{X}_1$ and $\mathbf{X}_2$ and any function $g$: $g(\mathbf{X}_1,\mathbf{X}_2)\in\mathbb{R}$, we have 
\begin{equation}
I(X_1, X_2) \geq \mathbb{E}[g(X_1, X_2)] - \mathbb{E}_{P(X_1)P(X_2)}\left[\exp(g(X_1, X_2) - 1)\right]\,.\label{2}
\end{equation}
The above lemma is used to $(Y;\mathbf{Z}^l)$. Plugging in $1+\log \frac{ \mbox{Cat}(\hat{y}_z^l )}{\mathbb{P}(Y)}$, the right hand side of Eq.(\ref{2}) is substituted by the cross-entropy loss, i.e.,
\begin{equation}
    I(Y; \mathbf{Z}^l) \rightarrow -\mathcal{L}_{CE}(\hat{y}^l_z,Y) \,. 
\end{equation}

The above lemma is also used to $(Y;\mathbf{X}_h^l \mid \mathbf{Z}^l)$. Plugging in $1+\log \frac{ \mbox{Cat}(\hat{y}_x^l )}{\mathbb{P}(Y)}$,  it can be transformed into the following cross-entropy loss, i.e.,
\begin{equation}
I(Y; \mathbf{X}_h^l\mid\mathbf{Z}^l) \rightarrow -\mathcal{L}_{CE}(\hat{y}^l_x,Y)\,.
\end{equation}
\noindent\textbf{Proof for Proposition \ref{prop2}}
According to the Markovian chain in Fig. \ref{fig:markov_chain}(a), we restate Proposition 2: For any distributions $\mathbb{Q}(\mathbf{Z}^{l})$ and $\mathbb{Q}(\mathbf{X}_h^{l})$, we have
\begin{align}
    I(\mathbf{X}^{l-1}, \mathbf{H};\mathbf{X}_h^{l}, \mathbf{Z}^l) 
    &\le I(\mathbf{X}^{l-1}, \mathbf{H};\mathbf{Z}^l) + I(\mathbf{Z}^{l}, \mathbf{H};\mathbf{X}_h^l) \notag \\
    &= \mathbb{E}\left( \log \frac{\mathbb{P}(\mathbf{Z}^l \mid \mathbf{X}^{l-1}, \mathbf{H})}{\mathbb{Q}(\mathbf{Z}^l)} \right) 
    - D_{KL}\left( \mathbb{P}(\mathbf{Z}^l ) \| \mathbb{Q}(\mathbf{Z}^l) \right) \notag \\
    & + \mathbb{E}\left( \log \frac{\mathbb{P}(\mathbf{X}_h^l \mid \mathbf{Z}^l, \mathbf{H})}{\mathbb{Q}(\mathbf{X}_h^l)} \right) 
    - D_{KL}\left( \mathbb{P}(\mathbf{X}_h^l) \| \mathbb{Q}(\mathbf{X}_h^l) \right) \notag \\
    &\le \mathbb{E}\left( \log \frac{\mathbb{P}(\mathbf{Z}^l \mid \mathbf{X}^{l-1}, \mathbf{H})}{\mathbb{Q}(\mathbf{Z}^l)} \right) 
    + \mathbb{E}\left( \log \frac{\mathbb{P}(\mathbf{X}_h^l \mid \mathbf{Z}^l, \mathbf{H})}{\mathbb{Q}(\mathbf{X}_h^l)} \right) \notag\\
    &=\mbox{ZIB}^l+\mbox{XIB}^l
\end{align}
\begin{equation}
    \mbox{ZIB}^l = D_{KL} \left( \mathbb{P}(\mathbf{Z}^{l} \mid \mathbf{X}^{l-1},\mathbf{H}) \parallel \mathbb{Q}(\mathbf{Z}^{l}) \right) , \mbox{XIB}^l = D_{KL} \left( \mathbb{P}(\mathbf{X}_h^{l} \mid \mathbf{Z}^{l},\mathbf{H}) \parallel 
    \mathbb{Q}(\mathbf{X}_h^{l}) \right)\,.
\end{equation}

\noindent\textbf{Proof for Proposition \ref{prop3}}
According to the Markovian chain in Fig. \ref{fig:markov_chain}(b), we restate Proposition 3: For any distribution
$\mathbb{Q}(\mathbf{X}^l)$, we have
\begin{align}
    I(\mathbf{X}_h^l,\mathbf{X}_t^l;\mathbf{X}^l) &= \mathbb{E}( \log \frac{\mathbb{P}(\mathbf{X}^l \mid \mathbf{X}_h^{l},\mathbf{X}_t^l)}{\mathbb{Q}(\mathbf{X}^l)} ) 
    - D_{KL}( \mathbb{P}(\mathbf{X}^l) \| \mathbb{Q}(\mathbf{X}^l) ) \notag\\
    &\le \mathbb{E}( \log \frac{\mathbb{P}(\mathbf{X}^l \mid \mathbf{X}_h^{l},\mathbf{X}_t^l)}{\mathbb{Q}(\mathbf{X}^l)} ) \notag \\
    &= D_{KL} ( \mathbb{P}(\mathbf{X}^{l} \mid \mathbf{X}_h^{l},\mathbf{X}_t^{l}) \parallel \mathbb{Q}(\mathbf{X}^{l}) \,.
\end{align}

\section{Additional Experiments.\label{app:appendixc}}
\subsection{Analysis of Attention Maps and Discriminative Connections}
To evaluate the contribution of the self-attention mechanism in the Transformer module, we visualize the average attention maps along the group level, as shown in Fig.\ref{fig:dis}. Additionally, to clearly display the variations within and between sub-networks, we divide the entire brain network into eight sub-networks: visual network (VN), auditory network (AN), basal ganglia network (BLN), default mode network (DMN), sensorimotor network (SMN), salience network (SN), motor network (MN), and cognitive control network (CCN).  In the ASD vs. NC classification task, the attention scores within the SMN are relatively high, indicating that the inter-ROI connectivity within the SMN plays a significant role in ASD prediction. This is consistent with previous studies that have shown reduced functional connectivity within the SMN \cite{ilioska2023connectome}, reflecting alterations in sensory and motor processing \cite{marco2011sensory}.  In the AD vs. NC classification task, the positions with higher attention scores are primarily located within the
DMN. Previous studies have indicated that the substantial accumulation of $\beta$-amyoid protein in the DMN of AD patients leads to reduced functional connectivity in this region, which is associated with cognitive impairments in AD patients \cite{myers2014within}.

\begin{figure}[t]
\includegraphics[width=\textwidth]{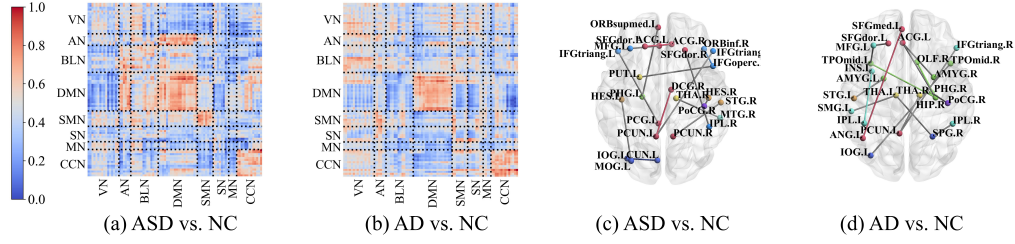}
\caption{Visualization of attention maps and discriminative connections. connections within the same neural system (\textcolor{VN}{VN}, \textcolor{AN}{AN}, \textcolor{BLN}{BLN}, \textcolor{DMN}{DMN}, \textcolor{SMN}{SMN}, \textcolor{SN}{SN}, \textcolor{MN}{MN}, \textcolor{CCN}{CCN}) are colored accordingly, while connections across different systems are colored gray.}
\label{fig:dis}
\end{figure}
We perform a t-test on attention maps to identify discriminative connections. For clarity, only connections with significant differences (p-value < 0.01) are shown in Fig. \ref{fig:dis}. In the ASD vs. NC classification task, the identified discriminative connections include those between the left anterior cingulate and paracingulate gyri (ACG.L) and right Inferior parietal (IPL.R), as well as between the right superior frontal gyrus, dorsolateral (SFGdor.R). This validates previous research, which suggests that changes in these connections are often associated with the repeat stereotypical behaviors commonly observed in individuals with ASD \cite{agam2010reduced}. In the AD vs. NC classification task, the discriminative connections include the connection between the left middle frontal gyrus (MFG.L) and the left superior frontal gyrus, dorsolateral (SFGdor.L). This is consistent with previous findings \cite{zhao2018evaluating}, which suggest that changes in functional connectivity between the right frontal lobe and the superior frontal gyrus in AD patients are associated with cognitive decline.
\subsection{Hyperparameter Analysis}
\begin{figure}[t]
\includegraphics[width=\textwidth]{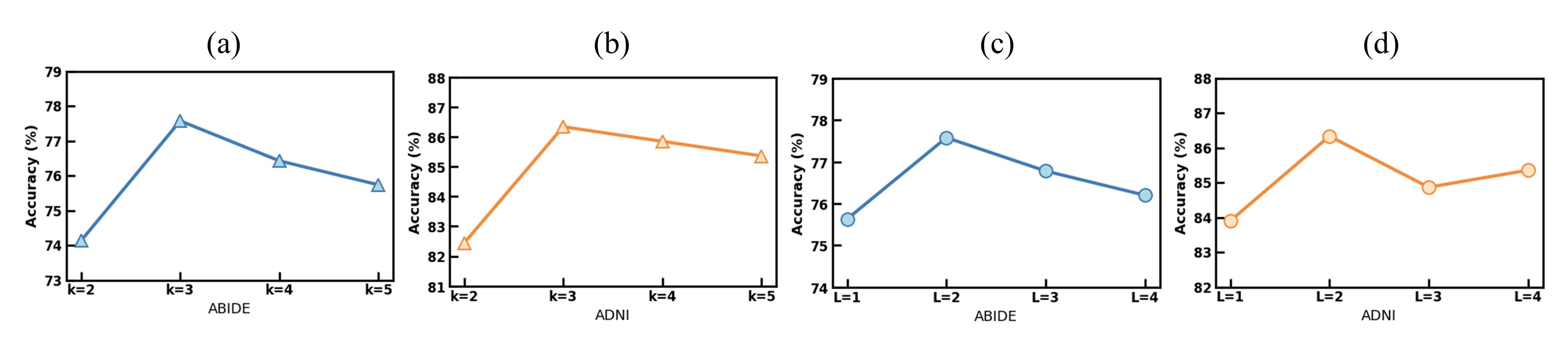}
\caption{The effect of varying hyperparameters. (a) and (c) are experiments on ABIDE dataset. (b) and (d) are experiments on ADNI dataset.}
\label{fig:sen}
\end{figure}
To evaluate the influence of hyperparameters on classification performance, we conduct controlled variable experiments on two main hyperparameters: the value of $k$ for KNN-based hypergraph construction and the number of layers $L$ in the information bottleneck-guided hypergraph Transformer. Fig.\ref{fig:sen}(a) and Fig.\ref{fig:sen}(b) illustrate the impact of different values of $k$ on two datasets. The accuracy is highest when $k = 3$ for both classification tasks. The accuracy is highest when $k = 3$ for both classification tasks. When $k = 2$, the hypergraph degenerates into a graph, causing a loss of high-order information in the brain network. When $k > 3$, the information aggregation range increases, and the representations of different nodes become more homogeneous, resulting in a drop in accuracy. 

Fig.\ref{fig:sen}(c) and Fig.\ref{fig:sen}(d) show the variation in accuracy with respect to the number of layers $L$ on two datasets. We find that when the number of layers $L=2$, the model achieved the highest accuracy in both classification tasks. When $L<2$, the model's depth was insufficient to effectively capture the complex dependencies within the brain network. On the other hand, when $L>2$, an excessive number of layers led to overfitting, resulting in a decline in performance. 

\subsection{Running Time}
\setlength{\tabcolsep}{2pt}
\begin{table*}[]
\scriptsize
    \caption{Running time with different methods.}
    \label{table:time}  %
    \centering   
    \resizebox{0.7\textwidth}{!}{
    \begin{tabular}{ccc}
    \toprule
    \multicolumn{1}{c}{{Method}} & \multicolumn{1}{c}{{Running Time on ABIDE(s)}} & \multicolumn{1}{c}{{Running Time on ADNI(s)}} \\ 
    \midrule
    BrainNetCNN & $39.75\pm1.03$ & $8.04\pm0.98$\\
    ContrastPool &$334.89\pm4.23$& $83.62\pm2.50$\ \\
    BrainGNN &$294.01\pm1.11$& $74.97\pm2.08$\ \\
    BrainIB &$212.43\pm1.32$& $51.66\pm1.58$\ \\
    $\mbox{HGNN}^+$ & $36.55\pm1.22$ &$7.63\pm3.22$ \\
    FC-HAT &$82.56\pm0.69$& $22.76\pm1.08$\ \\
    BrainNetTF & $49.72\pm1.02$ & $10.82\pm0.34$ \\
    Com-BrainTF & $61.26\pm2.31$& $14.32\pm0.74$\\
    ALTER &$54.87\pm0.98$ & $11.35\pm2.62$\\
    IBAHGT &$72.43\pm0.88$ & $15.97\pm0.49$  \\
    \bottomrule
    \end{tabular}}
    \label{time}
\end{table*}
Table \ref{time} shows the running time of IBAHGT compared to other methods on ABIDE and ADNI datasets. BrainGNN and ContrastPool are significantly slower than the other methods, primarily due to their pooling processes, which consume more time. BrainIB dynamically samples subgraphs during training and exhibits slower convergence, resulting in a longer training time. The hypergraph-based method FC-HAT is slower than both our IBAHGT and $\mbox{HGNN}^+$, mainly because FC-HAT dynamically updates the hypergraph structure at each layer using KNN and k-means methods, which adds additional time cost. The adaptive hypergraph convolution and node-level fusion in the proposed IBAHGT are formulated in matrix form during implementation to accelerate training on GPU. Although the training time is slightly longer compared to Transformer-based methods, this increase is acceptable given the significant performance improvement.

\subsection{Baselines and Additional Comparative Experiments on ADHD-200 Dataet}
The selected baselines mainly include three categories. The first category consists of convolutional neural network (CNN)-based method, such as BrainNetCNN \cite{kawahara2017brainnetcnn}. BrainNetCNN employs new feature extraction structures based on the topological locality characteristics of brain networks. The second category includes graph or hypergraph-based methods, such as ContrastPool \cite{xu2024contrastive}, BrainGNN \cite{li2021braingnn}, BrainIB \cite{zheng2024brainib}, $\mbox{HGNN}^{+}$ \cite{gao2022hgnn+}, and FC-HAT \cite{ji2022fc}. ContrastPool generates contrast graph through a dual-attention module to guide differentiable graph pooling, which is used to produce effective brain network representations for disease classification. BrainIB introduces the IB principle into brain network analysis to effectively identify disease-specific prominent brain network connections. $\mbox{HGNN}^{+}$ is a general hypergraph convolution method. FC-HAT combines KNN and k-means to dynamically construct hypergraphs and captures information within the hypergraph during the hypergraph aggregation attention phase. The third category consists of Transformer-based methods, including BrainNetTF \cite{kan2022brain}, Com-BrainTF \cite{bannadabhavi2023community}, and ALTER \cite{yu2024long}. BrainNetTF employs self-attention to learn the long-range dependencies between ROIs and utilizes orthogonal clustering for graph-level representation. Com-BrainTF is a hierarchical local-global transformer architecture that learns intra- and inter-community aware node embeddings. ALTER proposes a brain graph transformer to capture long-range dependencies between ROIs utilizing biased random walk. To ensure a fair comparison, we conduct experiments using the open-source code provided in the papers and the optimal parameters specified therein.

To further validate the performance of IBAHGT, the ADHD-200 dataset \cite{adhd2012adhd} is also used for comparative experiments. The ADHD-200 dataset is a multi-site neuroimaging resource for attention deficit hyperactivity disorder (ADHD) research, consisting of 357 ADHD patients and 573 normal controls (NCs) from eight international imaging centers. As shown in Table \ref{HD}, the experimental results demonstrate that IBAHGT outperforms the baselines on the majority of metrics, further validating the superiority of the proposed method in brain disease diagnosis.
{\setlength{\tabcolsep}{2pt}
\begin{table*}[t]
\scriptsize
    \caption{Comparison experiments results on the ADHD-200 dataset.}
    \label{tab:adhd}
    \centering
    \resizebox{0.65\linewidth}{!}{
        \begin{tabular}{cllll}
        \toprule
        \multicolumn{1}{c}{{Datasets (Tasks)}} & \multicolumn{4}{c}{{ADHD-200 (ADHD vs. NC)}} \\ 
        \cmidrule(lr){1-1} \cmidrule(lr){2-5}
        {Method} & ACC (\%) & SPE (\%) & SEN (\%) & AUC (\%) \\ 
        \midrule
        BrainNetCNN & 62.3$\pm$1.7 & 67.6$\pm$2.7 & 53.5$\pm$1.5 & 61.8$\pm$1.4 \\
        ContrastPool &61.9$\pm$1.6 & 69.0$\pm$4.0 & 50.4$\pm$4.2 & 60.6$\pm$3.5 \\
        BrainGNN & 62.5$\pm$1.9 & 67.3$\pm$3.3 & 54.7$\pm$4.5 & 62.3$\pm$2.1 \\
        BrainIB & 63.5$\pm$2.0 & 68.7$\pm$2.8 & 55.2$\pm$5.0 & 63.4$\pm$2.2 \\
        $\mbox{HGNN}^+$ & 64.4$\pm$1.0 & 68.9$\pm$2.5 & 57.2$\pm$5.8 & 64.1$\pm$1.6 \\
        FC-HAT & 66.5$\pm$1.1 & 70.1$\pm$2.0 & 60.6$\pm$4.5 & 66.3$\pm$1.2 \\
        BrainNetTF & 67.2$\pm$1.7 & 70.4$\pm$5.7 & 62.0$\pm$5.0 & 66.6$\pm$1.1 \\
        Com-BrainTF & 67.7$\pm$1.7 & 69.7$\pm$1.7 & 64.5$\pm$3.9 & 67.6$\pm$2.5 \\
        ALTER & 68.8$\pm$2.1 & \textbf{74.1$\pm$4.6} & 60.3$\pm$7.6 & 68.0$\pm$2.4 \\
        IBAHGT & \textbf{72.3$\pm$1.3}& 73.7$\pm$1.9 & \textbf{69.9$\pm$4.0} & \textbf{72.2$\pm$1.2} \\
        \bottomrule
        \end{tabular}}
        \label{HD}
\end{table*}

\section{Further Discussion.\label{app:appendixd}}
\subsection{Limitations.}
This study primarily focuses on functional brain networks constructed from fMRI data, providing a unified analysis of high-order relationships and both short- and long-range dependencies in brain network. However, the influence of structural brain networks constructed from DTI data on the diverse dependencies among ROIs in functional brain networks are also worthy of exploration. In future work, we will delve into the effects of structural connectivity on multiple types of dependencies in functional brain network and propose methods to integrate functional and structural brain networks for more comprehensive analyses of high-order correlations and both short- and long-range dependencies among ROIs.
\subsection{Possible Societal Impacts.}

This research involves brain disease diagnosis, and it is therefore necessary to acknowledge its potential societal impact. The proposed method contributes to the discovery of potential biomarkers, offering positive impacts for neuroscience research and computer-aided diagnosis. However, the use of artificial intelligence in diagnostic assistance may lead to misdiagnoses, which are difficult to completely avoid and could have negative impacts for both patients and society. In real-world clinical settings, AI-based methods should serve solely as decision-support tools, with final diagnostic decisions remaining the responsibility of medical professionals.

\end{document}